\documentclass[11pt]{article} %

\usepackage[pdfborder=0 0 0]{hyperref}
\usepackage{enumerate}

\usepackage{amsmath,amsfonts,amsthm}
\usepackage{latexsym}
\usepackage[margin=1in]{geometry}

\usepackage{xcolor}
\usepackage{datetime}

\usepackage{graphicx} %
\usepackage{subfigure} 

\usepackage{algorithm}
\usepackage{algorithmic}

\usepackage[round]{natbib}

\hypersetup{
  colorlinks,
  linkcolor={red!50!black},
  citecolor={blue!50!black},
  urlcolor={blue!80!black}
}

\newtheorem{lemma}{Lemma}

\newtheorem{theorem}{Theorem}

\newtheorem{Defn}{Definition}

\newcommand{\bbR}{{\mathbf{R}}}

\newcommand{\calX}{{\mathcal{X}}}

\newcommand{\bbE}{{\mathbb{E}}}

\newcommand{\step}{\ensuremath{\eta}}
\newcommand{\reg}{\ensuremath{\lambda}}
\newcommand{\Proj}[1]{\ensuremath{\Pi_\calW\left(#1\right)}}
\newcommand{\noi}[1]{\ensuremath{\gamma_{#1}}}
\newcommand{\Noi}[1]{\ensuremath{\Gamma_{#1}}}
\newcommand{\cs}[1]{\ensuremath{c_{#1}}}
\newcommand{\DC}{\ensuremath{D_{\mathsf{C}}}}
\newcommand{\DN}{\ensuremath{D_{\mathsf{N}}}}
\newcommand{\NoiC}{\ensuremath{\Noi{\mathsf{C}}}}
\newcommand{\NoiN}{\ensuremath{\Noi{\mathsf{N}}}}
\newcommand{\dfracC}{\ensuremath{\beta_{\mathsf{C}}}}
\newcommand{\dfracN}{\ensuremath{\beta_{\mathsf{N}}}}
\newcommand{\epsC}{\ensuremath{\epsilon_{\mathsf{C}}}}
\newcommand{\epsN}{\ensuremath{\epsilon_{\mathsf{N}}}}
\def\GC{\calG_{\mathsf{C}}}
\def\GN{\calG_{\mathsf{N}}}

\newcommand{\norm}[1]{\ensuremath{\left\| #1 \right\|}}

\newcommand{\CF}{\ensuremath{\mathsf{CF}}}
\newcommand{\NF}{\ensuremath{\mathsf{NF}}}

\newcommand{\MNIST}{\texttt{MNIST}}
\newcommand{\Covertype}{\texttt{Covertype}}

\def\calS{\mathcal{S}}
\def\calO{\mathcal{O}}
\def\calW{\mathcal{W}}
\def\calG{\mathcal{G}}
\def\calX{\mathcal{X}}
\def\bbR{\mathbb{R}}

\newcommand\esub[1]{\mathbb{E}_{#1}}
\def\calD{\mathcal{D}}

\def\argmin{\mathop{\mathbf{argmin}}}

\newcommand{\E}[1]{\ensuremath{\mathbb{E}\left[#1\right]}}
\newcommand{\Var}[1]{\ensuremath{\mathrm{Var}\left(#1\right)}}

\newcommand{\Esub}[3]{\ensuremath{\mathbb{E}_{#1,\dots, #2}\left[#3\right]}}

\newcommand{\bigO}[1]{\ensuremath{\mathrm{\mathcal{O}}\left(#1\right)}}

\title{Learning from Data with Heterogeneous Noise using SGD}

\author{
Shuang Song\thanks{
Computer Science and Engineering Dept., 
University of California, San Diego, 
\texttt{shs037@eng.ucsd.edu}}
\qquad 
Kamalika Chaudhuri\thanks{
Computer Science and Engineering Dept.,
University of California, San Diego,
\texttt{kamalika@cs.ucsd.edu}}
\qquad
Anand D. Sarwate\thanks{
Electrical and Computer Engineering Dept.,
Rutgers University,
\texttt{asarwate@ece.rutgers.edu}}
}
\date{\today}

\begin{document}

\maketitle

\begin{abstract}
We consider learning from data of variable quality that may be
obtained from different heterogeneous sources. Addressing learning from
heterogenous data in its full generality is a challenging problem.  In this
paper, we adopt instead a model in which data is observed through heterogeneous
noise, where the noise level reflects the quality of the data source. We study
how to use stochastic gradient algorithms to learn in this model. 
Our study is motivated by two concrete examples where this problem
arises naturally: learning with local differential privacy based on data from
multiple sources with different privacy requirements, and learning from data
with labels of variable quality.

The main contribution of this paper is to identify how heterogeneous noise
impacts performance. We show that given two datasets with heterogeneous noise,
the order in which to use them in standard SGD depends on  the learning rate.
We propose a method for changing the learning rate as a function of the
heterogeneity, and prove new regret bounds for our method in two cases of
interest. Experiments on real data show that our method performs better than
using a single learning rate and using only the less noisy of the two datasets
when the noise level is low to moderate. 
\end{abstract}

\section{Introduction}

Modern large-scale machine learning systems often integrate data from several
different sources.  In many cases, these sources provide data of a similar
type (i.e. with the same features) but collected under different circumstances.
For example, patient records from different studies of a particular drug may be
combined to perform a more comprehensive analysis, or a collection of images with
annotations from experts as well as non-experts may be combined to learn a
predictor. In particular, data from different sources may be of varying
\textit{quality}. In this paper we adopt a model in which
data is observed through heterogeneous noise, where the noise level reflects the
quality of the data source. We study how to use stochastic gradient
algorithms to learn from data of heterogeneous quality.

In full generality, learning from heterogeneous data is essentially the problem of
domain adaptation -- a challenge for which
good and complete solutions are difficult to obtain.  Instead, we focus on the
special case of heterogeneous noise and show how to use information about
the data quality to improve the performance of learning algorithms which ignore this
information.

Two concrete instances of this problem motivate our study: locally
differentially private learning from multiple sites, and classification with
random label noise.  Differential
privacy~\citep{DworkMNS:06sensitivity,DworkKMMN:06ourselves} is a privacy
model that has received significant attention in machine-learning and
data-mining applications. A variant of differential privacy is \textit{local
privacy}~ -- the learner can only access the data via noisy estimates, where
the noise guarantees privacy~\citep{DuchiJW:12nips_full,NIPS2013_5013}.  In
many applications, we are required to learn from sensitive data collected from
individuals with heterogeneous privacy preferences, or from multiple sites with
different privacy requirements; this results in the heterogeneity of noise
added to ensure privacy. Under random
classification noise (RCN)~\citep{kearns1998efficient}, labels are
randomly flipped before being presented to the algorithm. The heterogeneity in
the noise addition comes from combining labels of variable quality -- such as
labels assigned by domain experts with those assigned by a crowd.

To our knowledge, \citet{CKW05}~were the first to provide a theoretical study
of how to learn classifiers from data of variable quality. In their
formulation, like ours, data is observed through heterogeneous noise. Given
data with known noise levels, their study focuses on finding an optimal
ordering of the data and a stopping rule without any constraint on the
computational complexity.  We instead shift our attention to studying
{\em{computationally efficient strategies}} for learning classifiers from data
of variable quality.

We propose a model for variable data quality which is natural in the context of
large-scale learning using stochastic gradient descent (SGD) and its
variants~\citep{largeScaleSGD,langfordBook}.  We assume that the training data
are accessed through an oracle which provides an unbiased but noisy estimate of
the gradient of the objective. The noise comes from two sources: the random
sampling of a data point, and additional noise due to the data quality. 
Our two motivating applications -- learning with local differential privacy and
learning from data of variable quality -- can both be modeled as solving a
regularized convex optimization problem using SGD.  Learning from data with
heterogeneous noise in this framework thus reduces to running SGD with noisy
gradient estimates, where the magnitude of the added noise varies across
iterations.  

\textbf{Main results.} In this paper we study noisy stochastic gradient methods
when learning from multiple data sets with different noise levels. For simplicity 
we consider the case where there are two data sets, which we
call \textsf{Clean} and \textsf{Noisy}.  We process these data sets sequentially
using SGD with learning rate $\calO(1/t)$.  
In a future full version of this work we also analyze averaged gradient descent (AGD)
with learning rate $\calO(1/\sqrt{t})$. We address some basic questions in this setup:

\textit{In what order should we process the data?}  Suppose we use standard SGD on the union of \textsf{Clean} and \textsf{Noisy}. We show theoretically and empirically that the order in which we should process the datasets to get good performance depends on the learning rate of the algorithm: in some cases we should use the order $(\mathsf{Clean},\mathsf{Noisy})$ and in others $(\mathsf{Noisy},\mathsf{Clean})$.

\textit{Can we use knowledge of the noise rates?} We show that using separate learning rates that depend on the noise levels for the clean and noisy datasets improves the performance of SGD. We provide a heuristic for choosing these rates by optimizing an upper bound on the error for SGD that depends on the ratio of the noise levels. We analytically quantify the performance of our algorithm in two regimes of interest. For moderate noise levels, we demonstrate empirically that our algorithm outperforms using a single learning rate and using clean data only.

\textit{Does using noisy data always help?}  The work of \citet{CKW05} suggests that if the noise level of noisy data is above some threshold, then noisy data will not help. Moreover, when the noise levels are very high, our heuristic does not always empirically outperform simply using the clean data. On the other hand, our theoretical results suggest that changing the learning rate can make noisy data useful. How do we resolve this apparent contradiction?

We perform an empirical study to address this question. Our experiments demonstrate that very often, there exists a learning rate at which noisy data helps; however, because the actual noise level may be far from the upper bound used in our algorithm, our optimization may not choose the best learning rate for every data set.  We demonstrate that by adjusting the learning rate we can still take advantage of noisy data. 

For simplicity we, like previous work~\cite{CKW05}, assume that the algorithms know the noise levels exactly. However, our algorithms can still be applied in the presence of approximate knowledge of the noise levels, and our result on the optimal data order only needs to know which dataset has more noise.

\textbf{Related Work.} There has been significant work on the
convergence of SGD assuming analytic properties of the objective
function, such as strong convexity and smoothness. When the objective function
is $\lambda$-strongly convex, the learning rate used for SGD is
$\calO(1/\lambda t)$~\citep{NemirovskyYudin,AgarwalBRW:09nips,
RakhShamir:12arxiv,BachM:11nips}, which leads to a regret of
$\calO(1/\lambda^2 t)$ for smooth objectives. For non-smooth objectives, SGD with learning rate $\calO(1/\lambda t)$ followed by some form of averaging of the iterates achieves $\calO(1/\lambda
t)$~\citep{ConfidenceLevel,stochasticApproximation,StochConvexOpt,xiaoDA,DuchiSinger}.

There is also a body of literature on differentially private classification by
regularized convex optimization in the batch~\citep{CMS11,Rubinstein13,KST12}
as well as the online~\citep{JKT12} setting. In this paper, we consider
classification with {\em{local differential
privacy}}~\citep{WassermanZ:10framework,DuchiJW:12nips_full}, a stronger form
of privacy than ordinary differential privacy.  \citet{DuchiJW:12nips_full}
propose learning a classifier with local differential privacy using SGD, and
\citet{SCS13} show empirically that using mini-batches significantly improves
the performance of differentially private SGD. Recent work by
\citet{BassilyTS:14focs} provides an improved privacy analysis for non-local
privacy. Our work is an extension of these papers to heterogeneous privacy
requirements.

\citet{CKW05} study classification when the labels in each data set are
corrupted by RCN of different rates. Assuming the classifier minimizing the
empirical $0/1$ classification error can always be found, they propose a
general theoretical procedure that processes the datasets in increasing order
of noise, and determines when to stop using more data. In contrast, our noise
model is more general and we provide a polynomial time algorithm for learning.
Our results imply that in some cases the algorithm should process the
noisy data first, and finally, our algorithm uses all the data.

\section{The Model}
\label{sec:model}

We consider linear classification in the presence of noise. We are given $T$ labelled examples $(x_1, y_1), \ldots, (x_T, y_T)$, where $x_i \in \bbR^d$, and $y_i \in \{-1, 1\}$ and our goal is to find a hyperplane $w$ that largely separates the examples labeled $1$ from those labeled $-1$.  A standard solution is via the following regularized convex optimization problem:
\begin{align} \label{eqn:cvxopt}
w^* = \argmin_{w \in \calW} f(w) := \frac{\reg}{2} \|w\|^2 + \frac{1}{T} \sum_{i=1}^{T} \ell(w, x_i, y_i).
\end{align}
Here $\ell$ is a convex loss function, and $\frac{\reg}{2} \| w \|^2$ is a regularization term. Popular choices for $\ell$ include the logistic loss $\ell(w, x, y) = \log(1 + e^{-yw^{\top}x})$ and the hinge loss $\ell(w, x, y) = \max(0, 1 - y w^{\top} x)$.

Stochastic Gradient Descent (SGD) is a popular approach to solving \eqref{eqn:cvxopt}: starting with an initial $w_1$, at step $t$, SGD updates $w_{t+1}$ using the point $(x_t,y_t)$ as follows: 
\begin{align}
w_{t+1} = \Proj{w_t - \step_t (\reg w_t + \nabla \ell(w_t, x_t, y_t))}.
\end{align}
Here $\Pi$ is a projection operator onto the convex feasible set $\calW$, typically set to $\{w: \|w\|_2 \leq 1/\reg\}$ and $\step_t$ is a learning rate (or step size) which specifies how fast $w_t$ changes.  A common choice for the learning rate for the case when $\reg > 0$ is $c/t$, where $c = \Theta(1/\lambda)$. 

\subsection{The Heterogeneous Noise Model}  

We propose an abstract model for heterogeneous noise that can be specialized to two important scenarios: differentially private learning, and random classification noise.  By heterogeneous noise we mean that the distribution of the noise can depend on the data points themselves.  More formally, we assume that the learning algorithm may only access the labeled data through an oracle $\calG$ which, given a $w \in \bbR^d$, draws a fresh independent sample $(x, y)$ from the underlying data distribution, and returns an unbiased noisy gradient of the objective function $\nabla f(w)$, based on the example $(x, y)$:
\begin{align}
\E{\calG(w)}  = \reg w + \nabla \ell(w, x, y), \; \E{\| \calG(w) \|^2} \leq\Gamma^2. \label{eqn:noisemodel}
\end{align}
The precise manner in which $\calG(w)$ is generated depends on the application.  Define the \textit{noise level} for the oracle $\calG$ as the constant $\Gamma$ in \eqref{eqn:noisemodel}; larger $\Gamma$ means more noisy data. Finally, to model finite training datasets, we assume that an oracle $\calG$ may be called only a limited number of times.

Observe that in this noise model, we can easily use the noisy gradient returned by $\calG$ to perform SGD. The update rule becomes:
\begin{align}
\label{eqn:sgdupdate}
w_{t+1} = \Proj{w_t - \step_t \calG(w_t)}.
\end{align}
The SGD estimate is $w_{t+1}$. 

In practice, we can implement an oracle such as $\calG$ based on a finite labelled training set $D$ as follows. We apply a random permutation on the samples in $D$, and at each invocation, compute a noisy gradient based on the next sample in the permutation. The number of calls to the oracle is limited to $|D|$. If the samples in $D$ are drawn iid from the underlying data distribution, and if any extraneous noise added to the gradient at each iteration is unbiased and drawn independently, then this process will implement the oracle correctly. 

To model heterogeneous noise, we assume that we have access to two oracles $\calG_1$ and $\calG_2$ implemented based on datasets $D_1$ and $D_2$, which can be called at most $|D_1|$ and $|D_2|$ times respectively. For $j = 1, 2$, the noise level of oracle $\calG_j$ is $\Noi{j}$, and the values of $\Noi{1}$ and $\Noi{2}$ are known to the algorithm. In some practical situations, $\Noi{1}$ and $\Noi{2}$ will not be known exactly; however, our algorithm in Section~\ref{sec:t} also applies when approximate noise levels are known, and our algorithm in Section~\ref{sec:order} applies even when only the relative noise levels are known.  

\subsubsection{Local Differential Privacy}

Local differential privacy~\citep{WassermanZ:10framework,DuchiJW:12nips_full,KLNRS08} is a strong notion of privacy motivated by differential privacy~\citep{DworkMNS:06sensitivity}. An untrusted algorithm is allowed to access a perturbed version of a sensitive dataset through a sanitization interface, and must use this perturbed data to perform some estimation. The amount of perturbation is controlled by a parameter $\epsilon$, which measures the privacy risk.

\begin{Defn}[Local Differential Privacy]
\label{def:localdp}
Let $D = (X_1, \ldots, X_n)$ be a sensitive dataset where each $X_i \in \calD$ corresponds to data about individual $i$. A randomized sanitization mechanism $M$ which outputs a disguised version $(U_1, \ldots, U_n)$ of $D$ is said to provide $\epsilon$-local differential privacy to individual $i$, if for all $x, x' \in D$ and for all $S \subseteq \calS$, 
	\begin{align}
	\label{eq:localdp}
	\Pr(U_i \in S | X_i = x)\leq e^{\epsilon} {\Pr(U_i \in S | X_i = x')}.
	\end{align}
Here the probability is taken over the randomization in the sanitization mechanism, and $\epsilon$ is a parameter that measures privacy risk where smaller $\epsilon$ means less privacy risk.
\end{Defn}

Consider learning a linear classifier from a sensitive labelled dataset while ensuring local privacy of the participants. This problem can be expressed in our noise model by setting the sanitization mechanism as the oracle. Given a privacy risk $\epsilon$, for $w \in \bbR^d$, the oracle $\calG^{\text{DP}}$ draws a random labelled sample $(x, y)$ from the underlying data distribution, and returns the noisy gradient of the objective function at $w$ computed based on $(x, y)$ as
	\begin{align}
	\label{eq:dporacle}
	\calG^{\text{DP}}(w) = \reg w + \nabla \ell(w, x, y) + Z,
	\end{align}
where $Z$ is independent random noise drawn from the density: $\rho(z) \propto e^{-(\epsilon/2) \|z\|}$.

\citet{DuchiJW:12nips_full} showed that this mechanism provides $\epsilon$-local privacy assuming analytic conditions on the loss function, bounded data, and that the oracle generates a fresh random sample at each invocation. 
The following result shows how to set the parameters to fit in our heterogeneous noise model.  The full proof is provided in Appendix \ref{sec:modelproofs}.

\begin{theorem}
\label{thm:Glocaldp}
If $\| \nabla \ell(w, x, y) \| \leq 1$ for all $w$ and $(x, y)$, then $\calG^{\text{DP}}(w)$ is $\epsilon$-local differentially private. Moreover, for any $w$ such that $\|w\| \leq \frac{1}{\lambda}$, $\bbE[\calG^{\text{DP}}(w)] = \reg w + \nabla \bbE_{(x, y)}[\ell(w, x, y)]$, and
\[ \bbE[\|\calG^{\text{DP}}(w)\|^2] \leq 4 + \frac{4(d^2 + d)}{\epsilon^2}. \]
\end{theorem}
\begin{proof}\textit{(Sketch)}
The term 4 comes from upper bounding $\bbE[\|\reg w + \nabla \ell(w,x,y)\|^2]$ by $\max_{w,x,y} \|\reg w + \nabla \ell(w,x,y)\|^2$ using $\|w\| \leq 1/{\lambda}$ and $\| \nabla \ell(w, x, y) \| \leq 1$. The term ${4(d^2 + d)}/{\epsilon^2}$ comes from properties of the noise distribution.
\end{proof}

In practice, we may wish to learn classifiers from multiple sensitive datasets with different privacy parameters. For example, suppose we wish to learn a classifier from sensitive patient records in two different hospitals holding data sets $D_1$ and $D_2$, respectively.  The hospitals have different privacy policies, and thus different privacy parameters $\epsilon_1$ and $\epsilon_2$. This corresponds to a heterogeneous noise model in which we have two sanitizing oracles -- $\calG^{\text{DP}}_1$ and $\calG^{\text{DP}}_2$. For $j = 1, 2$, $\calG^{\text{DP}}_j$ implements a differentially private oracle with privacy parameter $\epsilon_j$ based on dataset $D_j$ and may be called at most $|D_j|$ times.

\subsubsection{Random Classification Noise}
In the random classification noise model of~\citet{kearns1998efficient}, the learning algorithm is presented with labelled examples $(x_1, \tilde{y}_1), \ldots, (x_T, \tilde{y}_T)$, where each $\tilde{y}_i \in \{ -1, 1\}$ has been obtained by independently flipping the {\em{true label}} $y_i$ with some probability $\sigma$. \citet{natarajan13} showed that solving
	\begin{align}\label{eqn:objrcn}
	\argmin_{w} \frac{\reg}{2} \|w\|^2 + \frac{1}{T} \sum_{i=1}^{T} \tilde{\ell}(w, x_i, \tilde{y}_i, \sigma)
\end{align}
yields a linear classifier from data with random classification noise, where $\tilde{\ell}$ is a surrogate loss function corresponding to a convex loss $\ell$: 
	\[
	\tilde{\ell}(w, x, y, \sigma) = \frac{(1 - \sigma)\ell(w, x, y) - \sigma \ell(w, x, -y)}{1 - 2 \sigma},
	\]
and $\sigma$ is the probability that each label is flipped. This problem can be expressed in our noise model using an oracle $\calG^{\mathrm{RCN}}$ which on input $w$ draws a fresh labelled example $(x, \tilde{y})$ and returns 
	\[
	\calG^{\mathrm{RCN}}(w) = \reg w + \nabla \tilde{\ell}(w, x, \tilde{y}, \sigma).
	\]
The SGD updates in~\eqref{eqn:sgdupdate} with respect to $\calG^{\mathrm{RCN}}$ minimize~\eqref{eqn:objrcn}. If $\|x\| \leq 1$ and $\|\nabla \ell(w, x, y)\|\leq 1$, we have $\E{\calG_{\mathrm{RCN}}(w)} = \reg w + \nabla {\ell}(w, x, y)$ and $\E{\|\calG^{\mathrm{RCN}}(w)\|^2_2} \leq 3 + {1}/{(1-2\sigma)^2}$, under the random classification noise assumption, so the oracle $\calG^{\mathrm{RCN}}$ satisfies the conditions in~\eqref{eqn:noisemodel} with $\Gamma^2 = 3 + {1}/{(1-2\sigma)^2}$. 

In practice, we may wish to learn classifiers from multiple datasets with different amounts of classification noise~\citep{CKW05}; for example, we may have a small dataset $D_1$ labeled by domain experts, and a larger noisier dataset $D_2$, labeled via crowdsourcing, with flip probabilities $\sigma_1$ and $\sigma_2$. We model this scenario using two oracles -- $\calG^{\mathrm{RCN}}_1$ and $\calG^{\mathrm{RCN}}_2$. For $j = 1, 2$, oracle $\calG^{\mathrm{RCN}}_j$ is implemented based on $D_j$ and flip probability $\sigma_j$, and may be called at most $|D_j|$ times. 

\section{Data order depends on learning rate}
\label{sec:order}

Suppose we have two oracles $\GC$ (for ``clean'') and $\GN$ (for ``noisy'') implemented based on datasets $\DC, \DN$ with noise levels $\NoiC, \NoiN$ (where $\NoiC < \NoiN$) respectively. In which order should we query the oracle when using SGD? Perhaps surprisingly, it turns out that the answer depends on the learning rate. Below, we show a specific example of a convex optimization problem such that with $\step_t = {c}/{t}$, the optimal ordering is to use $\GC$ first when $c \in (0, 1/\reg)$, and the optimal ordering is to use $\GN$ first when $c > 1/\reg$.

Let $|\DC| + |\DN| = T$ and consider the convex optimization problem:
\begin{equation}
\label{eqn:excvxopt}
\min_{w \in \calW} \frac{\reg}{2} \|w\|^2 - \frac{1}{T} \sum_{i=1}^{T} y_i w^{\top} x_i,
\end{equation}
where the points $\{(x_i, y_i)\}$ are drawn from the underlying distribution by $\GC$ or $\GN$. Suppose $\calG(w) = \reg w - y x + Z$ where $Z$ is an independent noise vector such that $\bbE[Z]=0$, $\bbE[\|Z\|^2] = V_{\mathsf{C}}^2$ if $\calG$ is $\GC$, and $\bbE[\|Z\|^2] = V_{\mathsf{N}}^2$ if $\calG$ is $\GN$ with $V_{\mathsf{N}}^2 \geq V_{\mathsf{C}}^2$.

For our example, we consider the following three variants of SGD: $\CF$ and $\NF$ for ``clean first'' and ``noisy first'' and $\mathsf{AO}$ for an ``arbitrary ordering'':

\begin{enumerate}
\item $\CF$: For $t \le |\DC|$, query $\GC$ in the SGD update \eqref{eqn:sgdupdate}. For $t > |\DC|$, query $\GN$.
\item $\NF$: For $t \le |\DN|$, query $\GN$ in the SGD update \eqref{eqn:sgdupdate}. For $t > |\DN|$, query $\GC$.
\item $\mathsf{AO}$: Let $S$ be an arbitrary sequence of length $T$ consisting of $|\DC|$ $\mathsf{C}$'s and $|\DN|$ $\mathsf{N}$'s. In the SGD update \eqref{eqn:sgdupdate} in round $t$, if the $t$-th element $S_t$ of $S$ is $\mathsf{C}$, then query $\GC$; else, query $\GN$. 
\end{enumerate}

In order to isolate the effect of the noise, we consider two additional oracles $\GC'$ and $\GN'$; the oracle $\GC'$ (resp. $\GN'$) is implemented based on the dataset $\DC$ (resp. $\DN$), and iterates over $\DC$ (resp. $\DN$) in exactly the same order as $\GC$ (resp. $\GN$); the only difference is that for $\GC'$ (resp. $\GN'$), no extra noise is added to the gradient (that is, $Z = 0$). The main result of this section is stated in Theorem~\ref{thm:order}. 

\begin{theorem}
\label{thm:order}
Let $\{w_t^{\mathsf{CF}}\}$, $\{w_t^{\mathsf{NF}}\}$ and $\{w_t^{\mathsf{AO}}\}$ be the sequences of updates obtained by running SGD for objective function \eqref{eqn:excvxopt} under $\CF$, $\NF$ and $\mathsf{AO}$ respectively, and let $\{v_t^{\mathsf{CF}}\}$, $\{v_t^{\mathsf{NF}}\}$ and $\{v_t^{\mathsf{AO}}\}$ be the sequences of updates under $\CF$, $\NF$ and $\mathsf{AO}$ with calls to $\GC$ and $\GN$ replaced by calls to $\GC'$ and $\GN'$. Let $T = |\DC| + |\DN|$.
\begin{enumerate}
\item If the learning rate $\step_t = {c}/{t}$ where $c \in (0,{1}/{\reg})$,  then 
	\[
	\E{\|v_{T+1}^{\mathsf{CF}} - w_{T+1}^{\mathsf{CF}}\|^2} \leq \E{\|v_{T+1}^{\mathsf{AO}} - w_{T+1}^{\mathsf{AO}}\|^2}.
	\]
\item If the learning rate $\step_t = {c}/{t}$ where $c > {1}/{\reg}$, then 
	\[
	\E{\|v_{T+1}^{\mathsf{NF}} - w_{T+1}^{\mathsf{NF}}\|^2} \leq \E{\|v_{T+1}^{\mathsf{AO}} - w_{T+1}^{\mathsf{AO}}\|^2}.
	\]
\end{enumerate}
\end{theorem}

\begin{proof}
Let the superscripts $\CF$, $\NF$ and $\mathsf{AO}$ indicate the iterates for the $\CF$, $\NF$ and $\mathsf{AO}$ algorithms.  
Let $w_1$ denote the initial point of the optimization. 
Let $(x_t^{\mathsf{O}}, y_t^{\mathsf{O}})$ be the data used under order $\mathsf{O}=\CF,\NF$ or $\mathsf{AO}$ to update $w$ at time $t$, $Z^{\mathsf{O}}_i$ be the noise added to the exact gradient by $\GC$ or $\GN$, depending on which oracle is used by ${\mathsf{O}}$ at $t$ and $w^{\mathsf{O}}_t$ be the $w$ obtained under order ${\mathsf{O}}$ at time $t$.
Then by expanding the expression for $w_t$ in terms of the gradients, we have
\begin{align}
w_{T+1}^{\mathsf{O}} =& w_1 \prod_{i=1}^{T} (1 - \step_t \reg) - \sum_{t=1}^{T} \step_t \left( \prod_{s=t+1}^{T} (1 - \step_s \reg) \right) ( y_t^{\mathsf{O}} x_t^{\mathsf{O}} + Z^{\mathsf{O}}_t ).
\label{eqn:worder}
\end{align}
Similarly, if $v_1 = w_1$, we have
	\begin{align}
	v_{T+1}^{\mathsf{O}} = w_1 \prod_{i=1}^{T} (1 - \step_t \reg) - \sum_{t=1}^{T} \step_t 
		\left( \prod_{s=t+1}^{T} (1 - \step_s \reg) \right) y_t^{\mathsf{O}} x_t^{\mathsf{O}} .
	\label{eqn:vorder}
	\end{align}
Define 
	\[
	\Delta_t = \step_t \prod_{s=t+1}^{\top} (1 - \step_s \reg).
	\]
Taking the expected squared difference between \eqref{eqn:worder} from \eqref{eqn:vorder}, we obtain
\begin{align}
	\E{\|v_{T+1}^{\mathsf{O}} - w_{T+1}^{\mathsf{O}}\|^2}	
	&= \E{ \norm{ \sum_{t=1}^{T} \step_t \left( \prod_{s=t+1}^{T} (1 - \step_s \reg) \right) Z^{\mathsf{O}}_t }^2} \notag \\
	&= \E{ \norm{ \sum_{t=1}^{T} \Delta_t  Z^{\mathsf{O}}_t }^2} \notag \\
	&= \sum_{t=1}^{T} \Delta_t^2 \E{\|Z^{\mathsf{O}}_t\|^2}, \label{eqn:dataorder_v-w}
\end{align}
where the second step follows because the $Z^{\mathsf{O}}_i$'s are independent. \\
If $\step_t = {c}/{t}$, then 
	\[ 
	\Delta_t  = \dfrac{c}{t}\prod_{s=t+1}^{\top} \left(1-\frac{c\reg}{s} \right).
	\]
Therefore
\begin{align*}
\dfrac{\Delta_{t+1}^2}{\Delta_t^2} 
=  \left(\dfrac{\dfrac{c}{t+1}\prod_{s=t+2}^{\top} \left( 1-\dfrac{c\reg}{s} \right)}{\dfrac{c}{t}\prod_{s=t+1}^{\top} \left( 1-\dfrac{c\reg}{s} \right)}\right)^2 
=  \left( \dfrac{t}{(t+1) \left( 1-\dfrac{c\reg}{t+1} \right)} \right)^2 
=  \left( \dfrac{1}{1+\dfrac{1-c\reg}{t}} \right)^2,
\end{align*}
which is smaller than $1$ if $c<1/\reg$, equal to $1$ if $c=1/\reg$, and greater than $1$ if $c>1/\reg$. Therefore $\Delta_t$ is decreasing if $c<1/\reg$ and is increasing if $c>1/\reg$.\\
If $\Delta_t$ is decreasing, then \eqref{eqn:dataorder_v-w} is minimized if $\E{\|Z^{\mathsf{O}}_t\|^2}$ is increasing; if $\Delta_t$ is increasing, then \eqref{eqn:dataorder_v-w} is minimized if $\E{\|Z^{\mathsf{O}}_t\|^2}$ is decreasing; and if $\Delta_t$ is constant, then \eqref{eqn:dataorder_v-w} is the same under any order of $\E{\|Z^{\mathsf{O}}_t\|^2}$. \\
Therefore for $c<1/\reg$, 
\begin{align*}
\E{ \norm{v_{T+1}^{\CF} - w_{T+1}^{\CF} }^2} &\leq \E{\norm{v_{T+1}^{\mathsf{AO}} - w_{T+1}^{\mathsf{AO}} }^2}  \leq \E{ \norm{v_{T+1}^{\NF} - w_{T+1}^{\NF} }^2}.
\end{align*}
For $c=1/\reg$,
\begin{align*}
\E{ \norm{v_{T+1}^{\CF} - w_{T+1}^{\CF} }^2} &= \E{\norm{v_{T+1}^{\mathsf{AO}} - w_{T+1}^{\mathsf{AO}} }^2} = \E{ \norm{v_{T+1}^{\NF} - w_{T+1}^{\NF} }^2}.
\end{align*}
For $c>1/\reg$,
\begin{align*}
\E{ \norm{v_{T+1}^{\CF} - w_{T+1}^{\CF} }^2} &\geq \E{\norm{v_{T+1}^{\mathsf{AO}} - w_{T+1}^{\mathsf{AO}} }^2} \geq \E{ \norm{v_{T+1}^{\NF} - w_{T+1}^{\NF} }^2}.
\end{align*}
\end{proof}

This result says that arbitrary ordering of the data is worse than sequentially processing one data set after the other except in the special case where $c = 1/\reg$. If the step size is small ($c < 1/\reg$), the SGD should use the clean data first to more aggressively proceed towards the optimum.  If the step size is larger ($c > 1/\reg$), then SGD should reserve the clean data for refining the initial estimates given by processing the noisy data.

\section{Adapting the learning rate to the noise level}
\label{sec:t}

We now investigate whether the performance of SGD can be improved by using different learning rates for oracles with different noise levels. Suppose we have oracles $\calG_1$ and $\calG_2$ with noise levels $\Noi{1}$ and $\Noi{2}$ that are implemented based on two datasets $D_1$ and $D_2$.  Unlike the previous section, we do not assume any relation between $\Noi{1}$ and $\Noi{2}$ -- we analyze the error for using oracle $\calG_1$ followed by $\calG_2$ in terms of $\Noi{1}$ and $\Noi{2}$ to choose a data order.  Let $T=|D_1|+|D_2|$. Let $\beta_1 = \frac{|D_1|}{T}$ and $\beta_2 = 1 - \beta_1 =\frac{|D_2|}{T}$ be the fraction of the data coming from $\calG_1$ and $\calG_2$, respectively.  We adapt the gradient updates in  \eqref{eqn:sgdupdate} to heterogeneous noise by choosing the learning rate $\step_t$ as a function of the noise level. Algorithm~\ref{alg:heterosgd} shows a modified SGD for heterogeneous learning rates.

\begin{algorithm}
\caption{SGD with varying learning rate}
\label{alg:heterosgd}
\begin{algorithmic}[1]
\STATE {\bf{Inputs:}} Oracles $\calG_1, \calG_2$ implemented by data sets $D_1, D_2$. Learning rates $\cs{1}$ and $\cs{2}$.
\STATE Set $w_1 = 0$.
\FOR{$t = 1, 2, \ldots, |D_1|$}
\STATE $w_{t+1} = \Proj{w_t - \frac{\cs{1}}{t} \calG_1(w_t)}$
\ENDFOR
\FOR{$t = |D_1| + 1, |D_1| + 2, \ldots, |D_1| + |D_2|$}
\STATE $w_{t+1} = \Proj{w_t - \frac{\cs{2}}{t}\calG_2(w_t)}$
\ENDFOR
\RETURN $w_{|D_1| + |D_2| +1}$.
\end{algorithmic}
\end{algorithm}

Consider SGD with learning rate $\step_t = \cs{1}/t$ while querying $\calG_1$ and with $\step_t = \cs{2}/t$ while querying $\calG_2$ in the update \eqref{eqn:sgdupdate}. We must choose an order in which to query $\calG_1$ and $\calG_2$ as well as the constants $\cs{1}$ and $\cs{2}$ to get the best performance.  We do this by minimizing an upper bound on the distance between the final iterate $w_{T+1}$ and the optimal solution $w^*$ to $\bbE[f(w)]$ where $f$ is defined in~\eqref{eqn:cvxopt}, and the expectation is with respect to the data distribution and the gradient noise; the upper bound we choose is based on~\citet{RakhShamir:12arxiv}. Note that for smooth functions $f$, a bound on the distance $\| w_{T+1} - w^*\|$ automatically translates to a bound on the regret $f(w_{T+1}) - f(w^*)$.

Theorem~\ref{lem:tub} generalizes the results of~\citet{RakhShamir:12arxiv} to our heterogeneous noise setting; the proof is in the supplement.

\begin{theorem}
\label{lem:tub}
If $2 \reg \cs{1} > 1$ and if $2 \reg \cs{2} \neq 1$, and if we query $\calG_1$ before $\calG_2$ with learning rates $\cs{1}/t$ and $\cs{2}/t$ respectively, then the SGD algorithm satisfies
	\begin{align} \label{eq:tub}
	& \E{\| w_{T+1} - w^* \|^2} 
	 \leq 	\frac{4\Noi{1}^2}{T} \cdot \dfrac{\beta_1^{2\reg \cs{2}-1} \cs{1}^2}{2\reg \cs{1} - 1} \nonumber \\ 
	&	+ 	\frac{4\Noi{2}^2}{T} \cdot \dfrac{(1 - \beta_1^{2\reg \cs{2}-1}) \cs{2}^2}{2\reg \cs{2}-1} 
		+ 	\calO \left( \frac{1}{T^{\min(2, 2 \reg \cs{1})}} \right).
	\end{align}
\end{theorem}
\begin{proof} (\textit{Sketch})
Let $g(w)$ be the true gradient $\nabla f(w)$ and $\hat g(w)$ be the unbiased noisy gradient provided by the oracle $\calG_1$ or $\calG_2$, whichever is queried.\\
By strong convexity of $f$, we have
\begin{align*}
\Esub{1}{t}{\|w_{t+1}-w^*\|^2}\le (1-2\reg \step_t)\Esub{1}{t}{\|w_{t}-w^*\|^2} + \step_t^2 \noi{t}^2.
\end{align*}
Solving it inductively, with $\noi{t} = \Noi{1}, \step_t=\cs{1}/t$ for $t \le \beta_1 T$ and $\noi{t} = \Noi{2}, \step_t=\cs{2}/t$ for $t > \beta_1 T$, we have
\begin{align*}
\Esub{1}{T}{\|w_{T+1}-w^*\|^2}
\le& \prod_{i=i_0}^{\beta_1 T} \left( 1-\dfrac{2\reg \cs{1}}{i} \right) \prod_{i=\beta_1 T+1}^{T} \left(1-\dfrac{2\reg \cs{2}}{i} \right)  \Esub{1}{T}{\|w_{i_0}-w^*\|^2} \\
&+ \Noi{1}^2 \prod_{i=\beta_1 T+1}^{T} \left(1-\dfrac{2\reg \cs{2}}{i} \right) \sum_{i=i_0}^{\beta_1 T}\dfrac{\cs{1}^2}{i^2} \prod_{j=i+1}^{\beta_1 T} \left(1-\dfrac{2\reg \cs{1}}{j} \right) \\
&+ \Noi{2}^2 \sum_{i=\beta_1 T+1}^{T}\dfrac{\cs{2}^2}{i^2} \prod_{j=i+1}^{T} \left(1-\dfrac{2\reg \cs{2}}{j} \right),
\end{align*}
where $i_0$ is the smallest positive integer such that $2\reg \step_{i_0}<1$, i.e, $i_0 = \lceil 2\cs{1}\reg \rceil$. \\
Then using $1-x \leq e^{-x}$ and upper bounding each term using integrals we get the \eqref{eq:tub}.
\end{proof}
Two remarks are in order. First, observe that the first two terms in the right hand side dominate the other term. Second, our proof techniques for Theorem~\ref{lem:tub}, adapted from~\cite{RakhShamir:12arxiv}, require that $2 \reg \cs{1} > 1$ in order to get a $O(1/T)$ rate of convergence; without this condition, the dependence on $T$ is $\Omega(1/T)$. 

\subsection{Algorithm description}
\label{sect:t_algDescrip}
Our algorithm for selecting $\cs{1}$ and $\cs{2}$ is motivated by Theorem~\ref{lem:tub}. We propose an algorithm that selects $\cs{1}$ and $\cs{2}$ by minimizing the quantity $B(c_1, c_2)$ which represents the highest order terms in Theorem~\ref{lem:tub}:
\begin{align}
B(c_1, c_2) = \frac{4 \Noi{1}^2 \beta_1^{2 \reg c_2 - 1} c_1^2}{T (2 \reg c_1 - 1)} + \frac{4 \Noi{2}^2 (1 - \beta_1^{2 \reg c_2 - 1}) c_2^2}{T( 2 \reg c_2 - 1) }.
\label{eq:orderopt_t}
\end{align}
Given $\reg$, $\Noi{1}$, $\Noi{2}$ and $\beta_1$, we use $c_1^*$ and $c_2^*$ to denote the values of $c_1$ and $c_2$ that minimize $B(c_1, c_2)$. We can optimize for fixed $c_2$ with respect to $c_1$ by minimizing $\frac{c_1^2}{2 \reg c_1 - 1}$; this gives $c_1^* = 1/\reg$, and $\frac{c_1^{*2}}{2\reg\cs{1}^* - 1} = 1/\reg^2$, which is independent of $\beta_1$ or the noise levels $\Noi{1}$ and $\Noi{2}$. Minimizing $B(c_1^*, c_2)$ with respect to $c_2$ can be now performed numerically to yield $c_2^* = \argmin_{c_2} B(c_1^*, c_2)$.  This yields optimal values of $\cs{1}$ and $\cs{2}$.

Now suppose we have two oracles $\GC$, $\GN$ with noise levels $\NoiC$ and $\NoiN$ that are implemented based on datasets $\DC$ and $\DN$ respectively. Let $\NoiC < \NoiN$, and let $\dfracC = \frac{|\DC|}{|\DC| + |\DN|}$ and $\dfracN = \frac{|\DN|}{|\DC| + |\DN|}$ be the fraction of the total data in each data set.  
Define the following functions:
\begin{align*}
H_{\mathsf{CN}}(c) = \frac{4\NoiC^2 \dfracC^{2\reg c - 1}}{\reg^2} + \frac{4\NoiN^2 (1 - \dfracC^{2\reg c - 1}) c^2}{2\reg c - 1},\\
H_{\mathsf{NC}}(c) = \frac{4\NoiN^2 \dfracN^{2\reg c - 1}}{\reg^2} + \frac{4\NoiC^2 (1 - \dfracN^{2\reg c - 1}) c^2}{2\reg c - 1}.
\end{align*}

These represent the constant of the leading term in the upper bound in Theorem~\ref{lem:tub} for $(\calG_1,\calG_2) = (\GC,\GN)$ and $(\calG_1,\calG_2) = (\GN,\GC)$, respectively.\\
Algorithm \ref{alg:tlearningrate} repeats the process of choosing optimal $\cs{1},\cs{2}$ with two orderings of the data -- $\GC$ first and $\GN$ first -- and selects the solution which provides the best bounds (according to the higher order terms of Theorem~\ref{lem:tub}). 

\begin{algorithm}
\caption{Selecting the Learning Rates}
\label{alg:tlearningrate}
\begin{algorithmic}[1]
\STATE {\bf{Inputs:}} Data sets $\DC$ and $\DN$ accessed through oracles $\GC$ and $\GN$ with noise levels $\NoiC$ and $\NoiN$.
\STATE Let $\dfracC = \frac{|\DC|}{|\DC| + |\DN|}$ and $\dfracN = \frac{|\DN|}{|\DC| + |\DN|}$. 
\STATE Calculate $c_{\mathsf{CN}} = \argmin_{c} H_{\mathsf{CN}}(c)$
and $c_{\mathsf{NC}} = \argmin_{c} H_{\mathsf{NC}}(c)$.
\IF{$H_{\mathsf{CN}}(c_{\mathsf{CN}}) \le H_{\mathsf{NC}}(c_{\mathsf{NC}})$}
\STATE Run Algorithm~\ref{alg:heterosgd} using oracles $(\GC,\GN)$, learning rates $\cs{1} = \frac{1}{\reg}$ and $\cs{2} = c_{\mathsf{CN}}$.
\ELSE
\STATE Run Algorithm~\ref{alg:heterosgd} using oracles $(\GN,\GC)$, learning rates $\cs{1} = \frac{1}{\reg}$ and $\cs{2} = c_{\mathsf{NC}}$.
\ENDIF
\end{algorithmic}
\end{algorithm}

\subsection{Regret Bounds}

To provide a regret bound on the performance of SGD with two learning rates, we need to plug the optimal values of $c_1$ and $c_2$ into the right hand side of~\eqref{eq:orderopt_t}. Observe that as $c_1 = c_2$ and $c_2 = 0$ are feasible inputs to~\eqref{eq:orderopt_t}, our algorithm by construction has a superior regret bound than using a single learning rate only, or using clean data only. 

Unfortunately, the value of $c_2$ that minimizes~\eqref{eq:orderopt_t} does not have a closed form solution, and as such it is difficult to provide a general simplified regret bound that holds for all $\Noi{1}$, $\Noi{2}$ and $\beta_1$. In this section, we consider two cases of interest, and derive simplified versions of the regret bound for SGD with two learning rates for these cases. 

We consider the two data orders $(\Noi{1},\Noi{2}) = (\NoiN,\NoiC)$ and $(\Noi{1},\Noi{2}) = (\NoiC,\NoiN)$ in a scenario where $\NoiN/\NoiC \gg 1$ and both $\dfracN$ and $\dfracC$ are bounded away from $0$ and $1$.  That is, the noisy data is much noisier.  The following two lemmas provide upper and lower bounds on $B(c_1^*, c_2^*)$ in this setting.

\begin{lemma}\label{lem:largenoi1}
Suppose $(\Noi{1},\Noi{2}) = (\NoiN,\NoiC)$ and $0 < \dfracN < 1$.  Then for sufficiently large $\NoiN/\NoiC$, the optimal solution $c_2^*$ to~\eqref{eq:orderopt_t} satisfies
	\begin{align*} 
2 c_2^* \reg \in & \Bigg[ 1 + \frac{2\log(\NoiN/\NoiC) + \log\log(1/\dfracN)}{\log (1/\dfracN)} , \\
&  1 + \frac{ 2 \log (4 \NoiN/\NoiC) + \log \log (1/\dfracN)}{\log(1/\dfracN)} \Bigg].
\end{align*}
Moreover, $B(c_1^*, c_2^*)$ satisfies:
	\begin{align*} 
	B(c_1^*, c_2^*) \geq & \frac{ 4\NoiC^2 (\log( \frac{\NoiN}{\NoiC}) 
		+ \frac{1}{2} \log \log \frac{1}{\dfracN})
		}{
		\reg^2 T \log(\frac{1}{\dfracN})
		} \\
	B(c_1^*, c_2^*) 
	\leq &
	\frac{4\NoiC^2} {\reg^2 T} \left( 4  %
		+ \frac{ 4 + 2 \log(\frac{\NoiN}{\NoiC}) + \log \log (\frac{1}{\dfracN})
		}{
		\log (\frac{1}{\dfracN})}\right).
	\end{align*}
\end{lemma}
\begin{proof}\textit{(Sketch)}
We prove that for any $2\reg\cs{2} \leq 1 + \frac{2\log(\NoiN/\NoiC) + \log\log(1/\dfracN)}{\log (1/\dfracN)}$, $B(\cs{1}^*, \cs{2})$ is decreasing with respect to $\cs{2}$ when $\NoiN/\NoiC$ is sufficiently large; and for any $2\reg\cs{2} \geq 1 + \frac{ 2 \log (4 \NoiN/\NoiC) + \log \log (1/\dfracN)}{\log(1/\dfracN)}$,  $B(\cs{1}^*, \cs{2})$ is increasing when $\NoiN/\NoiC$ is sufficiently large. Therefore the minimum of $B(\cs{1}^*, \cs{2})$ is achieved when $2\reg\cs{2}$ is in between. 
\end{proof}
Observe that the regret bound grows logarithmically with $\NoiN/\NoiC$. Moreover, if we only used the cleaner data, then the regret bound would be $\frac{4 \NoiC^2}{\reg^2 \dfracC T}$, which is better, especially for large $\NoiN/\NoiC$. This means that using two learning rates with the noisy data first gives poor results at high noise levels.  

Our second bound takes the opposite data order, processing the clean data first.

\begin{lemma}\label{lem:largenoi2}
Suppose $(\Noi{1},\Noi{2}) = (\NoiC,\NoiN)$ and $0 < \dfracC < 1$. Let $\sigma = (\NoiN/\NoiC)^{-2}$. Then for sufficiently large $\NoiN/\NoiC$, the optimal solution $c_2^*$ to~\eqref{eq:orderopt_t} satisfies: $2 c_2^* \reg \in \left[ \sigma, \frac{8}{\dfracC} \sigma \right]$. 
Moreover, $B(c_1^*, c_2^*)$ satisfies:
	\begin{align*}
	B(c_1^*, c_2^*) \geq & \frac{4 \NoiC^2}{ \reg^2 \dfracC T} \dfracC^{  8 \sigma/\dfracC } \\
	B(c_1^*, c_2^*) \leq & \frac{4 \NoiC^2}{ \reg^2 \dfracC T} 
		\dfracC^{ \sigma }
		\left( 1 + \sigma \frac{\log(1/\dfracC)}{4} \right).
	\end{align*} 
\end{lemma}
\begin{proof}\textit{(Sketch)}
Similar as the proof of Lemma \ref{lem:largenoi1}, we prove that for any $2\reg\cs{2} \leq  \sigma$, $B(\cs{1}^*, \cs{2})$ is decreasing with respect to $\cs{2}$; and for any $2\reg\cs{2} \geq \frac{8}{\dfracC} \sigma$,  $B(\cs{1}^*, \cs{2})$ is increasing when $\NoiN/\NoiC$ is sufficiently large. Therefore the minimum of $B(\cs{1}^*, \cs{2})$ is achieved when $2\reg\cs{2}$ is in between. 
\end{proof}

If we only used the clean dataset, then the regret bound would be $\frac{4 \NoiC^2}{ \reg^2 \dfracC T}$, so Lemma~\ref{lem:largenoi2} yields an improvement by a factor of $\dfracC^{ (\NoiN/\NoiC)^{-2}} \left( 1 + \left( \frac{\NoiN}{\NoiC} \right)^{-2} \frac{\log(1/\dfracC)}{4} \right)$. As $\dfracC < 1$, observe that this factor is always less than $1$, and tends to $1$ as $\NoiN/\NoiC$ tends to infinity; therefore the difference between the regret bounds narrows as the noisy data grows noisier. We conclude that using two learning rates with clean data first gives a better regret bound than using only clean data or using two learning rates with noisy data first.

\section{Experiments \label{sec:expts}}

We next illustrate our theoretical results through experiments on real data. We consider the task of training a regularized logistic regression classifier for binary classification under local differential privacy. For our experiments, we consider two real datasets -- \MNIST\ (with the task 1 vs. Rest) and \Covertype\ (Type 2 vs. Rest). The former consists of $60,000$ samples in $784$ dimensions, while the latter consists of $500,000$ samples in $54$-dimensions. We reduce the dimension of the \MNIST\ dataset to $25$ via random projections.  

To investigate the effect of heterogeneous noise, we divide the training data into subsets $(\DC,\DN)$ to be accessed through oracles $(\GC,\GN)$ with privacy parameters $(\epsC,\epsN)$ respectively. We pick $\epsC > \epsN$, so $\GN$ is noisier than $\GC$. To simulate typical practical situations where cleaner data is rare, we set the size of $\DC$ to be $\dfracC = 10\%$ of the total data size.  We set the regularization parameter $\lambda = 10^{-3}$, $\Gamma_C$ and $\Gamma_N$ according to Theorem~\ref{thm:Glocaldp} and use SGD with mini-batching (batch size $50$). 

\begin{figure*}[!t]
\vspace{-5pt}
\setlength{\abovecaptionskip}{0pt}
\centering
\subfigure[]{
\includegraphics[width=0.31\textwidth]{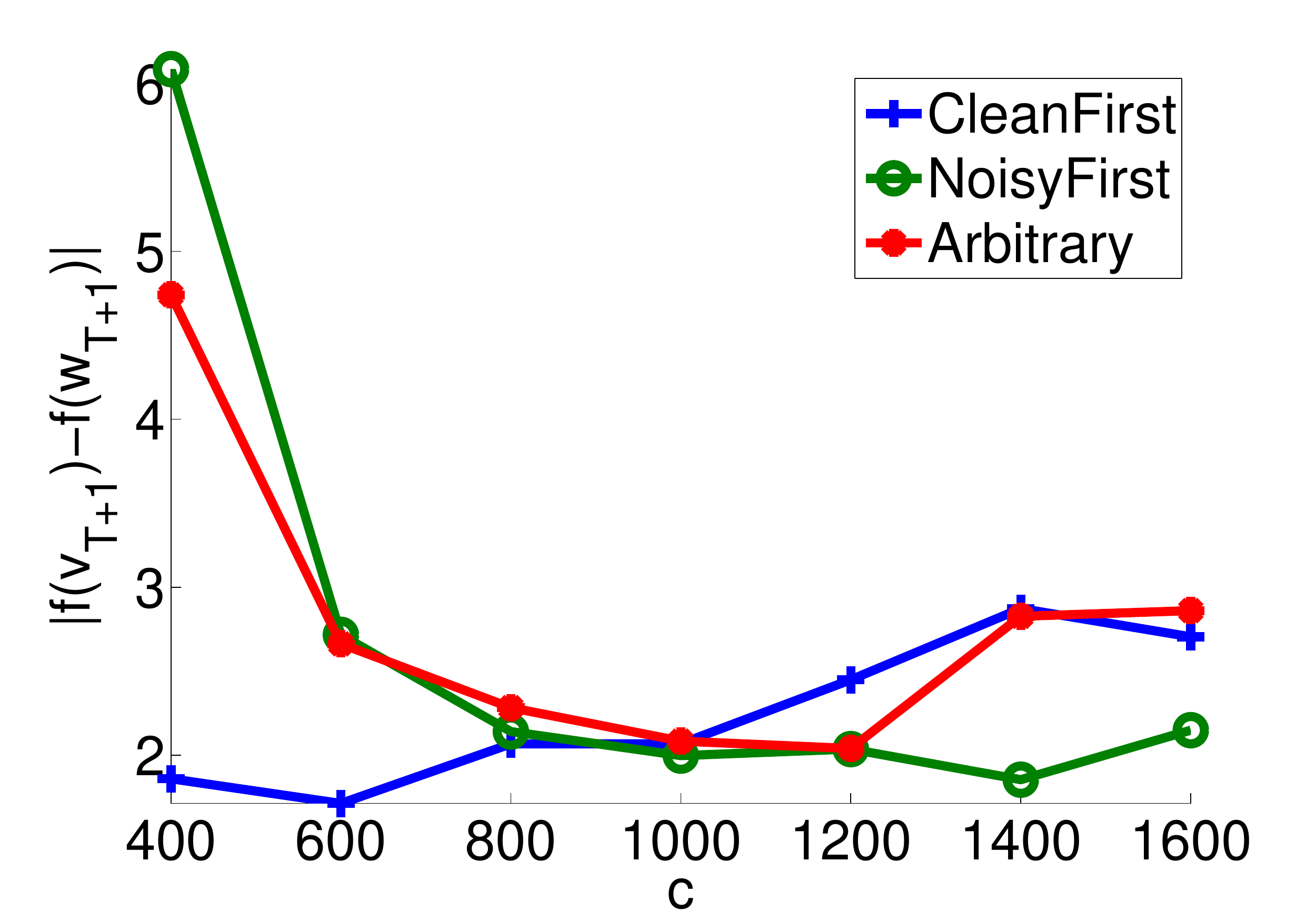}
\label{fig:ordermnist}}
\subfigure[]{
\includegraphics[width=0.31\textwidth]{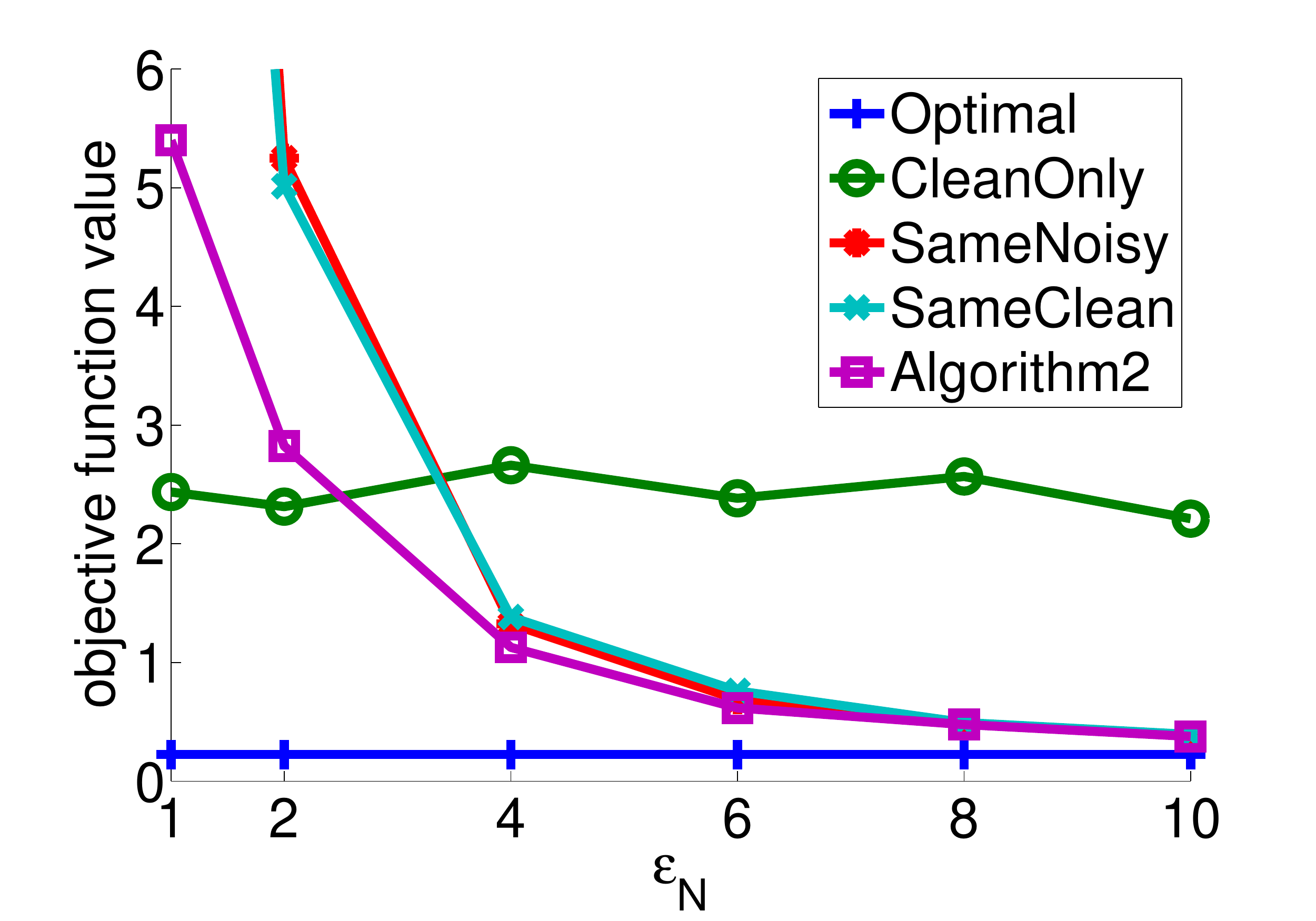}
\label{fig:tworatesa}}
\subfigure[]{
\includegraphics[width=0.31\textwidth]{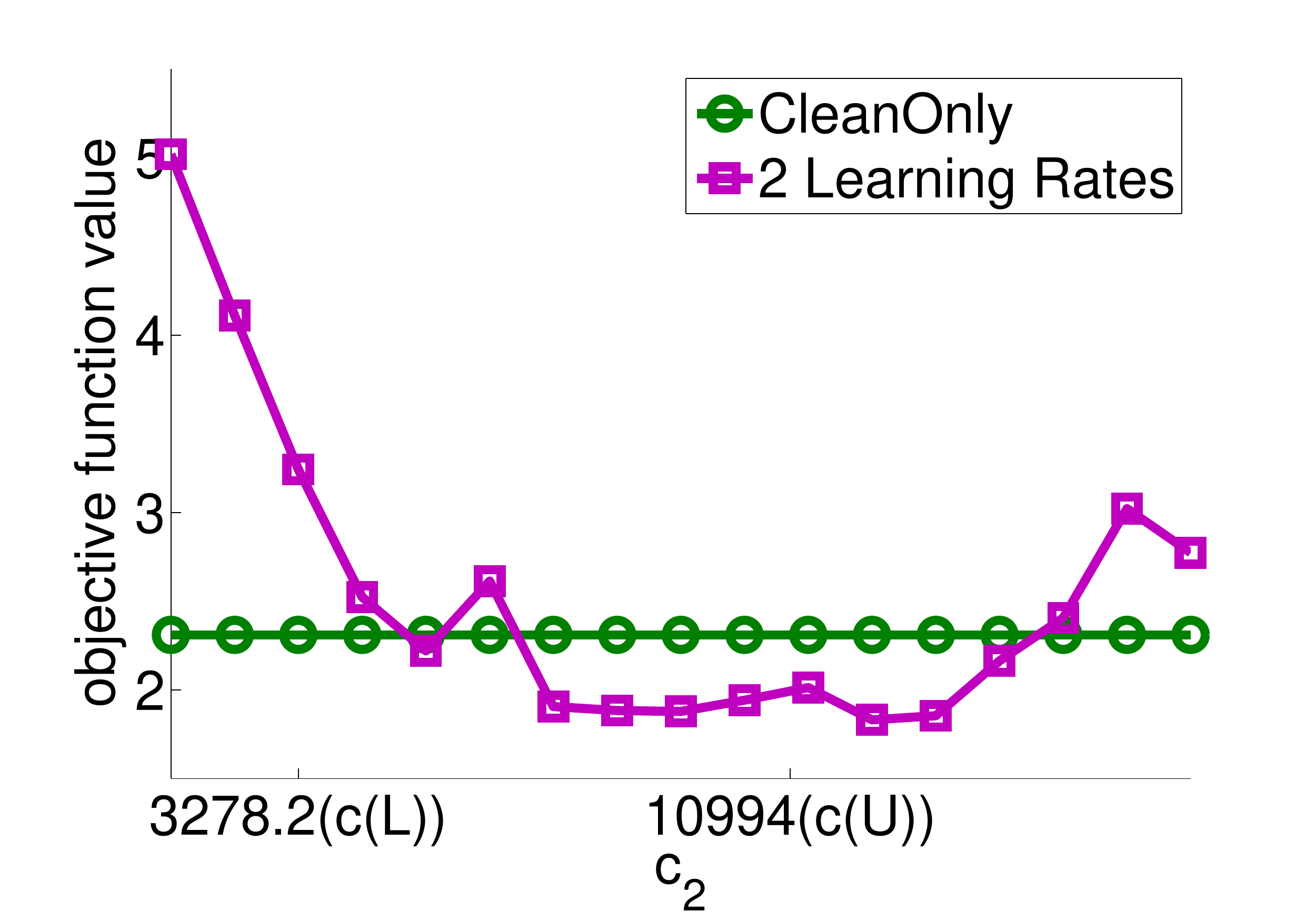}
\label{fig:varyca}}\\
\vspace{-0pt}
\subfigure[]{
\includegraphics[width=0.31\textwidth]{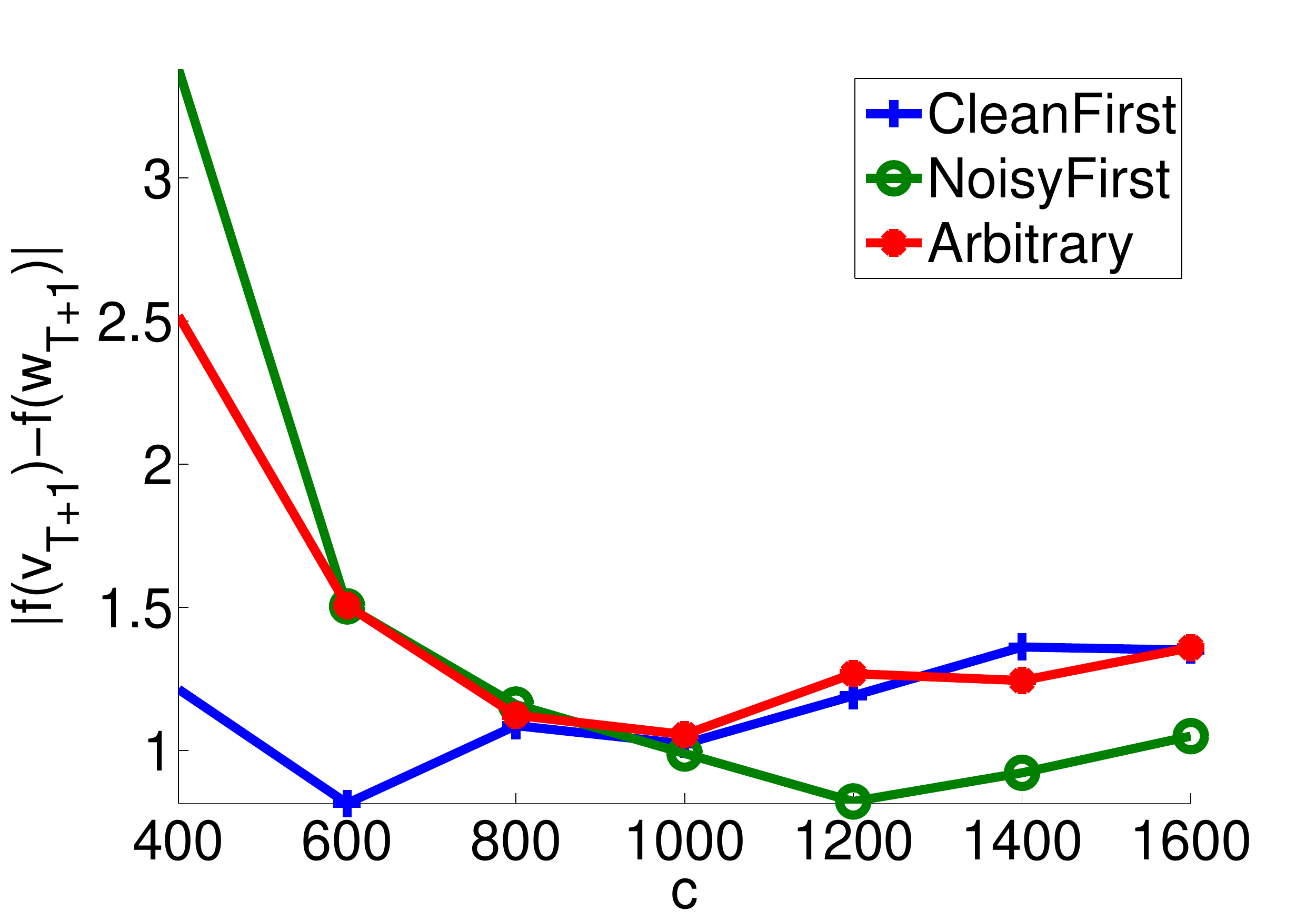}
\label{fig:ordercovtype}}
\subfigure[]{
\includegraphics[width=0.31\textwidth]{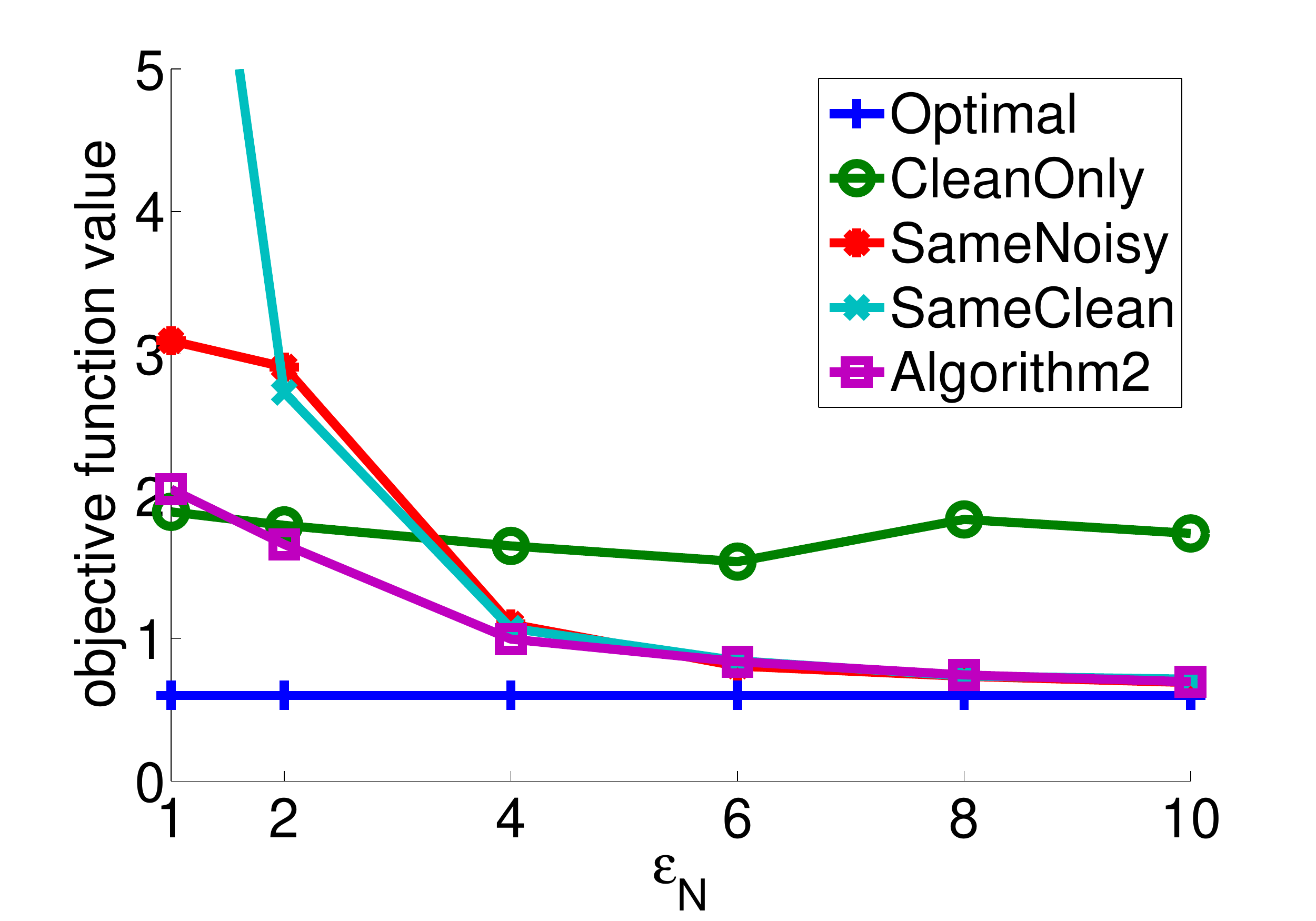}
\label{fig:tworatesb}}
\subfigure[]{
\includegraphics[width=0.31\textwidth]{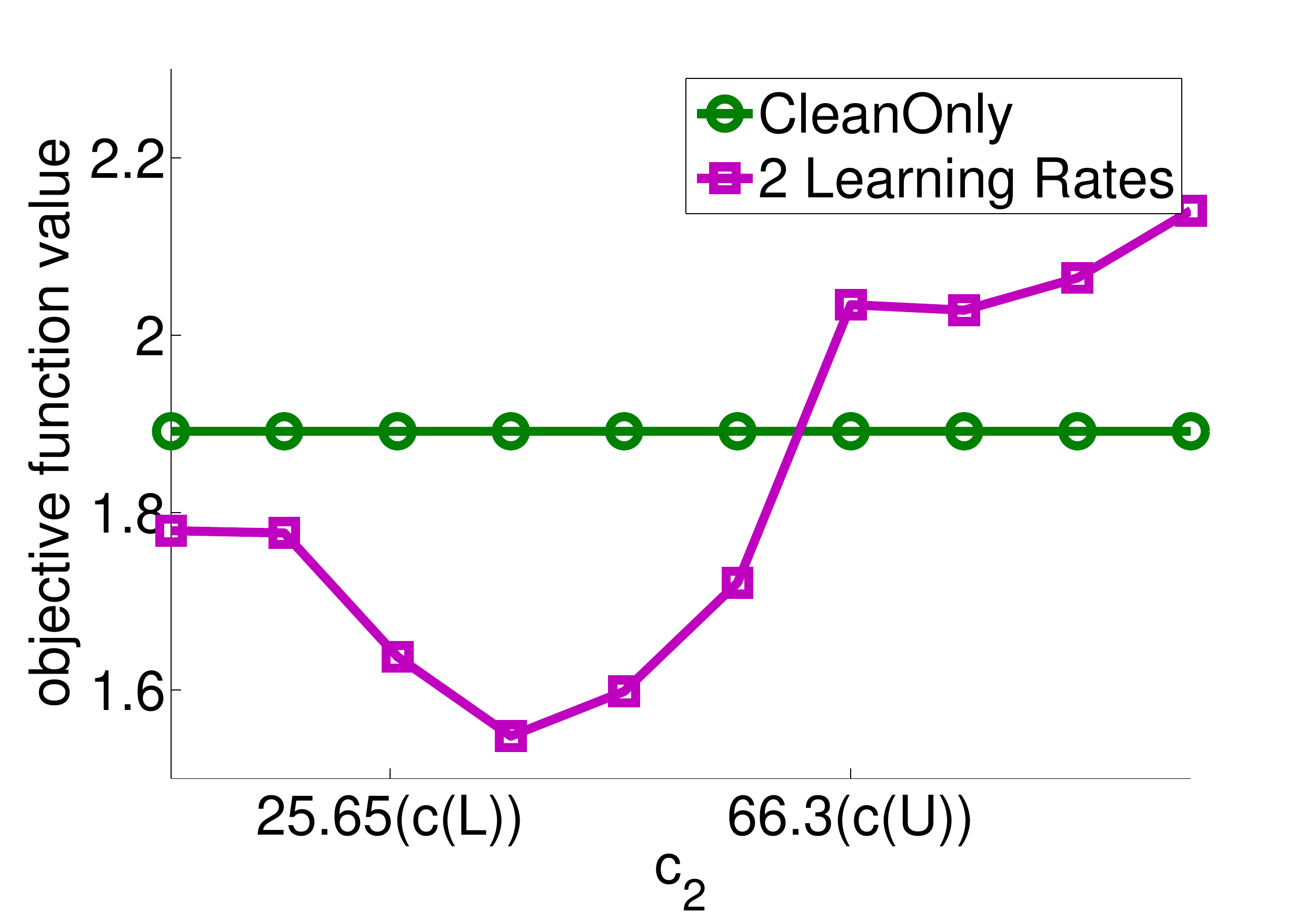}
\label{fig:varycb}}
\caption{Column 1 plots $|f(w_{T+1}) - f(v_{T+1})|$ vs. constant $c$ for $\reg = 0.001$. Column 2 plots final objective function value vs. $\epsilon_N$ for $\epsilon_C = 10$. Column 3 plots final objective function value vs. $c_2$ for $\epsilon_N = 2$ (top) and $\epsilon_N = 1$ (bottom). Top row shows figures for \MNIST\ and bottom row for \Covertype. }
\end{figure*}

\paragraph{Does Data Order Change Performance?} 
Our first task is to investigate the effect of data order on performance. For this purpose, we compare three methods -- $\mathsf{CleanFirst}$, where all of $\DC$ is used before $\DN$, $\mathsf{NoisyFirst}$, where all of $\DN$ is used before $\DC$, and $\mathsf{Arbitrary}$, where data from $\DN \cup \DC$ is presented to the algorithm in a random order.

The results are in Figures~\ref{fig:ordermnist} and~\ref{fig:ordercovtype}. We use $\epsilon_C=10, \epsilon_N = 3$. For each algorithm, we plot $|f(w_{T+1}) - f(v_{T+1})|$ as a function of the constant $c$ in the learning rate. Here $f(w_{T+1})$ is the function value obtained after $T$ rounds of SGD, and $f(v_{T+1})$ is the function value obtained after $T$ rounds of SGD if we iterate over the data in the same order, but add no extra noise to the gradient. (See Theorem~\ref{thm:order} for more details.) As predicted by Theorem \ref{thm:order}, the results show that for $c < \frac{1}{\lambda}$, $\mathsf{CleanFirst}$ has the best performance, while for $c > \frac{1}{\lambda}$, $\mathsf{NoisyFirst}$ performs best. $\mathsf{Arbitrary}$ performs close to $\mathsf{NoisyFirst}$ for a range of values of $c$, which we expect as only $10\%$ of the data belongs to $\DC$.

\paragraph{Are Two Learning Rates Better than One?} We next investigate whether using two learning rates in SGD can improve performance. We compare five approaches. $\mathsf{Optimal}$ is the gold standard where we access the raw data without any intervening noisy oracle. $\mathsf{CleanOnly}$ uses only $\DC$ with learning rate with the optimal value of $c$ obtained from Section~\ref{sec:t}. $\mathsf{SameClean}$ and $\mathsf{SameNoisy}$ use a single value of the constant $c$ in the learning rate for $\DN \cup \DC$, where $c$ is obtained by optimizing \eqref{eq:orderopt_t}\footnote{Note that we plug in separate noise rates for $\GC$ and $\GN$ in the learning rate calculations.} under the constraint that $\cs{1}=\cs{2}$. $\mathsf{SameClean}$ uses all of $\DC$ before using $\DN$, while $\mathsf{SameNoisy}$ uses all of $\DN$ before using $\DC$. In $\mathsf{Algorithm2}$, we use Algorithm~\ref{alg:tlearningrate} to set the two learning rates and the data order ($\DC$ first or $\DN$ first). In each case, we set $\epsC = 10$, vary $\epsN$ from $1$ to $10$, and plot the function value obtained at the end of the optimization. 

The results are plotted in Figures~\ref{fig:tworatesa} and~\ref{fig:tworatesb}. Each plotted point is an average of $100$ runs. It is clear that $\mathsf{Algorithm2}$, which uses two learning rates, performs better than both $\mathsf{SameNoisy}$ and $\mathsf{SameClean}$. As expected, the performance difference diminishes as $\epsN$ increases (that is, the noisy data gets cleaner). For moderate and high $\epsN$, $\mathsf{Algorithm2}$ performs best, while for low $\epsN$ (very noisy $\DN$), $\mathsf{CleanOnly}$ has slightly better performance. We therefore conclude that using two learning rates is better than using a single learning rate with both datasets, and that $\mathsf{Algorithm2}$ performs best for moderate to low noise levels.

\paragraph{Does Noisy Data Always Help?}  A natural question to ask is whether using noisy data always helps performance, or if there is some threshold noise level beyond which we should not use noisy data. Lemma~\ref{lem:largenoi2} shows that in theory, we obtain a better upper bound on performance when we use noisy data; in contrast, Figures~\ref{fig:tworatesa} and~\ref{fig:tworatesb} show that for low $\epsN$ (high noise), $\mathsf{Algorithm2}$ performs worse than $\mathsf{CleanOnly}$. How do we explain this apparent contradiction?

To understand this effect, in Figures~\ref{fig:varyca} and~\ref{fig:varycb} we plot the performance of SGD using two learning rates (with $c_1 = \frac{1}{\lambda}$) against $\mathsf{CleanOnly}$ as a function of the second learning rate $c_2$. The figures show that the best performance is attained at a value of $c_2$ which is different from the value predicted by $\mathsf{Algorithm2}$, and {\em{this best performance is better than $\mathsf{CleanOnly}$}}. Thus, noisy data always improves performance; however, the improvement may not be achieved at the learning rate predicted by our algorithm. 

Why does our algorithm perform suboptimally? We believe this happens because the values of $\Gamma_N$ and $\Gamma_C$ used by our algorithm are fairly loose upper bounds. For local differential privacy, an easy lower bound on $\Gamma$ is $\sqrt{\frac{4(d^2 + d)}{\epsilon^2 b}}$, where $b$ is the mini-batch size; let $c_2(L)$ (resp. $c_2(U)$) be the value of $c_2$ obtained by plugging in these lower bounds (resp. upper bounds from Theorem~\ref{thm:Glocaldp}) to Algorithm 1. Our experiments show that the optimal value of $c_2$ always lies between $c_2(L)$ and $c_2(U)$, which indicates that the suboptimal performance may be due to the looseness in the bounds.

We thus find that even in these high noise cases, theoretical analysis often allows us to identify {\em{an interval}} containing the optimal value of $c_2$. In practice, we recommend running Algorithm 2 twice -- once with upper, and once with lower bounds to obtain an interval containing $c_2$, and then performing a line search to find the optimal $c_2$.%

\section{Conclusion \label{sec:conclusion}}
In this paper we propose a model for learning from heterogeneous noise that is appropriate for studying stochastic gradient approaches to learning. In our model, data from different sites are accessed through different oracles which provide noisy versions of the gradient. Learning under local differential privacy and random classification noise are both instances of our model. We show that for two sites with different noise levels, processing data from one site followed by the other is better than randomly sampling the data, and the optimal data order depends on the learning rate. We then provide a method for choosing learning rates that depends on the noise levels and showed that these choices achieve lower regret than using a common learning rate. We validate these findings through experiments on two standard data sets and show that our method for choosing learning rates often yields improvements when the noise levels are moderate.  In the case where one data set is much noisier than the other, we provide a different heuristic to choose a learning rate that improves the regret.

There are several different directions towards generalizing the work here.  Firstly, extending the results to multiple sites and multiple noise levels will give more insights as to how to leverage large numbers of data sources.  This leads naturally to cost and budgeting questions: how much should we pay for additional noisy data?  Our results for data order do not depend on the actual noise levels, but rather their relative level.  However, we use the noise levels to tune the learning rates for different sites. If bounds on the noise levels are available, we can still apply our heuristic. Adaptive approaches for estimating the noise levels while learning are also an interesting approach for future study.

\paragraph{Acknowledgements.} The work of K. Chaudhuri and S. Song was sponsored by NIH under U54 HL108460 and the NSF under IIS 1253942. %

\appendix

\section{Appendix}
\newcommand{\calGdp}{\calG_\text{DP}}

\subsection{Mathematical miscellany}

In many cases we would like to bound a summation using an integral.
\begin{lemma} \label{lem:sum_int}
For $x \geq 0$, we have
\begin{equation} \label{eqn:sum_int_inc_leq}
\sum_{i=a}^b i^x \leq \int_a^{b+1} i^x di = \dfrac{(b+1)^{x+1} - a^{x+1}}{x+1}
\end{equation}
\begin{equation} \label{eqn:sum_int_inc_geq}
\sum_{i=a}^b i^x \geq \int_{a-1}^{b} i^x di = \dfrac{b^{x+1} - (a-1)^{x+1}}{x+1}
\end{equation}

For $x < 0$ and $x \neq -1$, we have
\begin{equation} \label{eqn:sum_int_dec_leq}
\sum_{i=a}^b i^x \leq \int_{a-1}^{b} i^x di = \dfrac{b^{x+1} - (a-1)^{x+1}}{x+1}
\end{equation}
\begin{equation} \label{eqn:sum_int_dec_geq}
\sum_{i=a}^b i^x \geq \int_{a}^{b+1} i^x di = \dfrac{(b+1)^{x+1} - a^{x+1}}{x+1}
\end{equation}

For $x = -1$, we have
\begin{equation} \label{eqn:sum_int_dec_leq_log}
\sum_{i=a}^b i^x \leq \int_{a-1}^{b} i^x di = \log \frac{b}{a - 1} 
\end{equation}
\begin{equation} \label{eqn:sum_int_dec_geq_log}
\sum_{i=a}^b i^x \geq \int_{a}^{b+1} i^x di = \log \frac{b+1}{a}
\end{equation}

\end{lemma}

The sequence $\{i^x\}$ is increasing when $x>0$ and is decreasing when $x<0$.  The proof follows directly from applying standard technique of bounding summation with integral. 

\subsection{Details from Section~\ref{sec:model} \label{sec:modelproofs}}
\begin{proof} (Of Theorem~\ref{thm:Glocaldp})
Consider an oracle $\calG$ implemented based on a dataset $D$ of size $T$. Given any sequence $w_1, w_2, \ldots, w_T$, the {\em{disguised version}} of $D$ output by $\calG$ is the sequence of gradients $\calG(w_1), \ldots, \calG(w_T)$. Suppose that the oracle accesses the data in a (random) order specified by a permutation $\pi$; for any $t$, any $x, x' \in \calX$, $y, y' \in \{1,-1\}$, we have 
\begin{align*}
\dfrac{\rho(\calG(w_t) = g | (x_{\pi(t)}, y_{\pi(t)}) = (x,y))}{\rho(\calG(w_t) = g | (x_{\pi(t)}, y_{\pi(t)}) = (x',y'))}
=&		\dfrac{\rho(Z_t = g - \reg w - \nabla \ell(w, x, y))}{\rho(Z_t = g - \reg w - \nabla \ell(w, x', y'))} \\
=&		\dfrac{e^{-(\epsilon/2) \|g - \reg w - \nabla \ell(w, x, y)\|} }{e^{-(\epsilon/2) \|g - \reg w - \nabla \ell(w, x', y')\|}} \\
\leq&	\exp\left( (\epsilon/2) (\|\nabla \ell(w, x, y)\|+\|\nabla \ell(w, x', y')\|) \right) \\
\leq&	\exp\left( \epsilon \right).
\end{align*}
The first inequality follows from the triangle inequality, and the last step follows from the fact that $\| \nabla \ell(w, x, y) \| \leq 1$. The privacy proof follows.

For the rest of the theorem, we consider a slightly generalized version of SGD that includes mini-batch updates. Suppose the batch size is $b$; for standard SGD, $b = 1$. For a given $t$, we call $\calG(w_t)$ $b$ successive times to obtain noisy gradient estimates $g_1(w_t), \ldots, g_b(w_t)$; these are gradient estimates at $w_t$ but are based on separate (private) samples. The SGD update rule is:
\[ w_{t+1} = \Proj{w_t - \frac{\step_t}{b} (g_1(w_t) + \ldots + g_b(w_t))}. \]

For any $i$, $\bbE[g_i(w_t)] = \lambda w + \bbE[ \nabla \ell(w, x, y)]$, where the first expectation is with respect to the data distribution and the noise, and the second is with respect to the data distribution; the unbiasedness result follows.  

We now bound the norm of the noisy gradient calculated from a batch. Suppose that the oracle accesses the dataset $D$ in an order $\pi$. Then, $g_i(w_t) = \reg w + \nabla \ell(w_t, x_{\pi( (t-1)b + i)}, y_{\pi( (t - 1)b + i)}) + Z_{(t - 1)b + i}$. Expanding on the expression for the expected squared norm of the gradient, we have
\begin{align}
\E{\norm{\frac{1}{b}(g_1(w_t) + \ldots + g_b(w_t))}^2} 
=&	\E{\norm{\reg w + \dfrac{1}{b}\sum_{i=1}^{b} \nabla \ell(w_t, x_{\pi( (t - 1)b + i)}, y_{\pi( (t-1)b + i)})}^2} \notag \\ 
& + \dfrac{2}{b}\E{\left(\reg w + \dfrac{1}{b}\sum_{i=1}^{b} \nabla \ell(w_t, x_{\pi( (t-1)b + i)}, y_{ \pi( (t-1)b + i) }) \right) \cdot \left(\sum_{i=1}^{b} Z_{ (t - 1)b + i} \right)} \notag \\
&   + \dfrac{1}{b^2}\E{\norm{\sum_{i=1}^{b} Z_{ (t-1)b + i }}^2} 
	\label{eq:oraclenorm}
\end{align}

We now look at the three terms in \eqref{eq:oraclenorm} separately.  
\\
The first term can be further expanded to:
\begin{align}
\E{ \norm{ \reg w }^2} & + \E{ \norm{ \frac{1}{b^2} \sum_{i=1}^{b} \nabla \ell(w_t, x_{\pi( (t - 1)b + i)}, y_{\pi( (t-1)b + i)})}^2} \notag \\
& + 2 \reg w \cdot \left( \sum_{i=1}^{b} \E{\nabla \ell(w_t, x_{\pi( (t - 1)b + i)}, y_{\pi( (t-1)b + i)})} \right) \label{eqn:normbound1}
\end{align}
The first term in \eqref{eqn:normbound1} is at most $\reg^2 \max_{w \in \calW} \| w\|^2$, which is at most $1$. The second term is at most $\max_{w} \reg \| w\| \cdot \max_{w, x, y} \| \nabla \ell(w, x, y)\| \leq 1$, and the third term is at most $2$. Thus, the first term in \eqref{eq:oraclenorm} is at most $4$. Notice that this upper bound can be pretty loose compare to the average $\norm{\reg w + \dfrac{1}{b}\sum_{i=1}^{b} \nabla \ell(w_t, x_{\pi( (t - 1)b + i)}, y_{\pi( (t-1)b + i)})}^2$ values seen in experiment. This leads to a loose estimation of the noise level for oracle $\calG^{\text{DP}}$.
\\
To bound the second term in \eqref{eq:oraclenorm}, observe that for all $i$, $Z_{(t-1)b + i}$ is independent of any $Z_{(t-1)b + i'}$ when $i \neq i'$, as well as of the dataset. Combining this with the fact that $\E{Z_{\tau}} = 0$ for any $\tau$, we get that this term is $0$.
\\
To bound the third term in \eqref{eq:oraclenorm}, we have:
\begin{align*}
\dfrac{1}{b^2}\E{\norm{\sum_{t \in B} Z_t }^2_2} &=	\dfrac{1}{b^2}\E{\sum_{t \in B} \norm{Z_t }^2_2 + \sum_{t \in B, s\in B, t\neq s} Z_t \cdot Z_s}\\
&=	\dfrac{1}{b^2}\sum_{t \in B} \E{\norm{Z_t }^2_2} + \dfrac{1}{b^2} \sum_{t \in B, s\in B, t\neq s} \E{Z_t} \cdot \E{Z_s}\\
&=	\dfrac{1}{b^2}\sum_{t \in B} \E{\norm{Z_t }^2_2},
\end{align*}
where the first equality is from the linearity of expectation and the last two equalities is from the fact that $Z_i$ is independently drawn zeros mean vector. Because $Z_t$ follows $\rho(Z_t = z) \propto e^{-(\epsilon/2) \|z\|}$, we have 
	\[
	\rho(\|Z_t\| = x) \propto x^{d-1} e^{-(\epsilon/2) x},
	\]
which is a Gamma distribution. For $X \sim \text{Gamma}(k, \theta)$, $\E{X} = k\theta$ and $\Var{X} = k\theta^2$. Also, by property of expectation, $\E{X^2} = (\E{X})^2+\Var{X}$. We then have $\E{\norm{Z_t }^2_2} = \dfrac{4(d^2+d)}{\epsilon^2}$ and the whole term equals to $\dfrac{4(d^2+d)}{\epsilon^2 b}$.\\
Combining the three bounds together, we have a final bound of  $4+\dfrac{4(d^2+d)}{\epsilon^2 b}$. The lemma follows.

\end{proof}

\subsection{Proofs from Section~\ref{sec:t}}
Recall that we have oracles $\calG_1, \calG_2$ based on data sets $D_1$ and $D_2$. The fractions of data in each data set are $\beta_1 = \frac{|D_1|}{|D_1|+|D_2|}$ and  $\beta_2 = \frac{|D_2|}{|D_1|+|D_2|}$, respectively.

\subsubsection{Proof of Theorem \ref{lem:tub}}
Theorem \ref{lem:tub} is a corollary of the following Lemma.
\begin{lemma} \label{lem:full_tub}
Consider the SGD algorithm that follows Algorithm \ref{alg:heterosgd}. Suppose the objective function is $\reg$-strongly convex, and define $\calW = \{w: \|w\| \leq B\}$. If $2\reg \cs{1} > 1$ and $i_0 = \lceil 2\cs{1}\reg \rceil$, then we have the following two cases:
\begin{enumerate}
\item If $2\reg \cs{2} \neq 1$, 
\begin{align*}
\E{\|w_{t+1} - w^*\|^2} &\leq \left(4\Noi{1}^2 \dfrac{\beta_1^{2\reg \cs{2}-1} \cs{1}^2}{2\reg \cs{1} - 1}
	+ 4\Noi{2}^2 \dfrac{\cs{2}^2 (1 - \beta_1^{2\reg \cs{2}-1})}{2\reg \cs{2}-1} \right) \cdot \dfrac{1}{T} + \calO\left(\frac{1}{T^{\min(2 \reg \cs{1}, 2)}}\right) 
\end{align*}
\item If $2\reg \cs{2} = 1$, 
\begin{align*}
\E{\|w_{t+1} - w^*\|^2} 
&\le \left( 4\Noi{1}^2 \dfrac{\beta_1^{2\reg \cs{2}-1} \cs{1}^2}{2\reg \cs{1} - 1}
		+ 4\Noi{2}^2 \cs{2}^2 \log \dfrac{1}{\beta_1} \right) \cdot \dfrac{1}{T} + \calO\left(\frac{1}{T^{\min(2 \reg \cs{1}, 2)}}\right)
\end{align*}
\end{enumerate}
\end{lemma}

We first begin with a lemma which follows from arguments very similar to those made in~\cite{RakhShamir:12arxiv}. 

\begin{lemma}
Let $w^*$ be the optimal solution to $\bbE[f(w)]$. Then,
\begin{align*}
\Esub{1}{t}{\|w_{t+1}-w^*\|^2} \le (1-2\reg \step_t)\Esub{1}{t}{\|w_{t}-w^*\|^2} + \step_t^2 \noi{t}^2.
\end{align*}
where the expectation is taken wrt the oracle as well as sampling from the data distribution.
\label{lem:rakhlinextended}
\end{lemma}

\begin{proof}(Of Lemma~\ref{lem:rakhlinextended})
By strong convexity of $f$, we have 
\begin{equation} \label{eqn:strongconvex}
f(w') \geq f(w) + g(w)^{\top}(w' - w) + \dfrac{\reg}{2} \|w - w'\|^2.
\end{equation}
Then by taking $w=w_t$, $w'=w^*$ we have
\begin{equation} \label{eqn:strongconv1}
g(w_t)^{\top}(w_t-w^*) \geq f(w_t) - f(w^*) + \dfrac{\reg}{2}\|w_t-w^*\|^2.
\end{equation}
And similarly by taking $w'=w_t$, $w=w^*$, we have
\begin{equation} \label{eqn:strongconv2}
f(w_t) - f(w^*) \geq \dfrac{\reg}{2}\|w_t-w^*\|^2.
\end{equation}
By the update rule and convexity of $\calW$, we have
\begin{align*}
	\Esub{1}{t}{\|w_{t+1}-w^*\|^2} &=
		\Esub{1}{t}{\|\Proj{w_{t}-\step_t \hat{g}(w_t)}-w^*\|^2} \notag \\
	&\le \Esub{1}{t}{\|w_{t}-\step_t \hat{g}(w_t)-w^*\|^2} \notag \\
	&= \Esub{1}{t}{\|w_t-w^*\|^2} - 2\step_t \Esub{1}{t}{\hat{g}(w_t)^{\top}(w_t-w^*)}%
	\step_t^2 \Esub{1}{t}{\|\hat{g}(w_t)\|^2}.
\end{align*}
Consider the term $\Esub{1}{t}{\hat{g}(w_t)^{\top}(w_t-w^*)}$, where the expectation is taken over the randomness from time $1$ to $t$. Since $w_t$ is a function of the samples used from time $1$ to $t-1$, it is independent of the sample used at $t$. So we have%
\begin{align*}
\Esub{1}{t}{\|w_{t+1}-w^*\|^2}
& \le \Esub{1}{t}{\hat{g}(w_t)^{\top}(w_t-w^*)} 
	\notag \\
&=\Esub{1}{t-1}{\esub{t} [\hat{g}(w_t)^{\top}(w_t-w^*) | w_t]} 
	\notag \\
&=\Esub{1}{t-1}{\esub{t} [\hat{g}(w_t)^{\top} | w_t](w_t-w^*)}
	\notag \\
&=\Esub{1}{t-1}{g(w_t)^{\top} (w_t-w^*)}.
\end{align*}
We have the following upper bound:
\begin{align*}
\Esub{1}{t}{\|w_{t+1}-w^*\|^2} 
	&\le \Esub{1}{t}{\|w_t-w^*\|^2} - 2\step_t \Esub{1}{t-1}{g(w_t)^{\top}(w_t-w^*)} 
	\notag \\
	&\hspace{15pt} + \step_t^2 \Esub{1}{t}{\|\hat{g}(w_t)\|^2}.
\end{align*}
By \eqref{eqn:strongconv1} and the bound $\E{\|\hat{g}(w_t)\|^2} \leq \noi{t}^2$, we have%
\begin{align*}
&\Esub{1}{t}{\|w_{t+1}-w^*\|^2} \le  \Esub{1}{t}{\|w_t-w^*\|^2} - 2\step_t \Esub{1}{t-1}{f(w_t)-f(w^*) +\dfrac{\reg}{2} \|w_t-w^*\|^2} 
	+ \step_t^2 \noi{t}^2.
\end{align*}
Then by \eqref{eqn:strongconv2} and the fact that $w_t$ is independent of the sample used in time $t$, we have the following recursion: %
\begin{align*}
\Esub{1}{t}{\|w_{t+1}-w^*\|^2} \le (1-2\reg \step_t)\Esub{1}{t}{\|w_{t}-w^*\|^2} + \step_t^2 \noi{t}^2.
\end{align*}
\end{proof}

\begin{proof} (Of Lemma~\ref{lem:full_tub})
Let $g(w)$ be the true gradient $\nabla f(w)$ and $\hat g(w)$ be the unbiased noisy gradient provided by the oracle $\calG_1$ or $\calG_2$, whichever is queried. From Lemma~\ref{lem:rakhlinextended}, we have the following recursion:
\begin{align*}
\Esub{1}{t}{\|w_{t+1}-w^*\|^2} \le (1-2\reg \step_t)\Esub{1}{t}{\|w_{t}-w^*\|^2} + \step_t^2 \noi{t}^2.
\end{align*}

Let $i_0$ be the smallest positive integer such that $2\reg \step_{i_0}<1$, i.e, $i_0 = \lceil 2\cs{1}\reg \rceil$. Notice that for fixed step size constant $c$ and $\reg$, $i_0$ would be a fixed constant. Therefore we assume that $i_0 < \beta T$. Using the above inequality inductively, and substituting $\noi{t} = \Noi{1}$ for $t \le \beta_1 T$ and $\noi{t} = \Noi{2}$ for $t > \beta_1 T$, we have
\begin{align*}
\Esub{1}{T}{\|w_{T+1}-w^*\|^2} &\leq	 \prod_{i=i_0}^{\beta_1 T} \left(1-2\reg \step_i \right) \prod_{i=\beta_1 T+1}^{T} \left(1-2\reg \step_i \right) \Esub{1}{T}{\|w_{i_0}-w^*\|^2} \\
&\hspace{15pt} +  \Noi{1}^2 \prod_{i=\beta_1 T+1}^{T} (1-2\reg \step_i) \sum_{i=i_0}^{\beta_1 T}\step_i^2 \prod_{j=i+1}^{\beta_1 T} (1-2\reg \step_j) \\
&\hspace{15pt} +  \Noi{2}^2 \sum_{i=\beta_1 T+1}^{T} \step_i^2 \prod_{j=i+1}^{T} (1-2\reg \step_j).
\end{align*}
By substituting $\step_t=\dfrac{\cs{1}}{t}$ for $D_1$ and $\step_t=\dfrac{\cs{2}}{t}$ for $D_2$, we have
\begin{align*}
\Esub{1}{T}{\|w_{T+1}-w^*\|^2} 
&\le \prod_{i=i_0}^{\beta_1 T} \left( 1-\dfrac{2\reg \cs{1}}{i} \right) \prod_{i=\beta_1 T+1}^{T} \left(1-\dfrac{2\reg \cs{2}}{i} \right) \Esub{1}{T}{\|w_{i_0}-w^*\|^2} \\
&\hspace{15pt} +  \Noi{1}^2 \prod_{i=\beta_1 T+1}^{T} \left(1-\dfrac{2\reg \cs{2}}{i} \right) \sum_{i=i_0}^{\beta_1 T}\dfrac{\cs{1}^2}{i^2} \prod_{j=i+1}^{\beta_1 T} \left(1-\dfrac{2\reg \cs{1}}{j} \right) \\
&\hspace{15pt} +  \Noi{2}^2 \sum_{i=\beta_1 T+1}^{T}\dfrac{\cs{2}^2}{i^2} \prod_{j=i+1}^{T} \left(1-\dfrac{2\reg \cs{2}}{j} \right).
\end{align*}
Applying the inequality $1 - x \leq e^{-x}$ to each of the terms in the products, and simplifying, we get:
\begin{align}
\Esub{1}{T}{\|w_{T+1}-w^*\|^2} 
&\le  e^{-2\reg \cs{1} \sum_{i=i_0}^{\beta_1 T} \frac{1}{i}} e^{-2\reg \cs{2} \sum_{i=\beta_1 T+1}^{\top} \frac{1}{i}} \Esub{1}{T}{\|w_{i_0}-w^*\|^2} \notag \\
	&\hspace{15pt} +  \Noi{1}^2 e^{-2\reg \cs{2} \sum_{i=\beta_1 T+1}^{\top} \frac{1}{i}} \sum_{i=i_0}^{\beta_1 T}\dfrac{\cs{1}^2}{i^2} e^{-2\reg \cs{1} \sum_{j=i+1}^{\beta_1 T} \frac{1}{j}} \notag \\
	&\hspace{15pt} +  \Noi{2}^2 \sum_{i=\beta_1 T+1}^{\top}\dfrac{\cs{2}^2}{i^2} e^{-2\reg \cs{2}  \sum_{j=i+1}^{\top} \frac{1}{j}}.
	\label{eq:tvar_ub1}
\end{align}	
We would like to bound \eqref{eq:tvar_ub1} term by term.\\
A bound we will use later is: %
\begin{equation} \label{eq:bd_e_longterm}
e^{2\reg \cs{2} /\beta_1 T} = 1+\dfrac{2\reg \cs{2}}{\beta_1 T} e^{2\reg \cs{2} /\beta_1 T'} \leq 1+\dfrac{2\reg \cs{2}}{\beta_1 T} e^{2\reg \cs{2} /\beta_1},
\end{equation}
where the equality is obtained using Taylor's theorem, and the inequality follows because $T'$ is in the range $[1, \infty)$.  Now we can bound the three terms in \eqref{eq:tvar_ub1} separately.

\paragraph{The first term in \eqref{eq:tvar_ub1}:} We bound this as follows: 
\begin{align*}
&			e^{-2\reg \cs{1} \sum_{i=i_0}^{\beta_1 T} \frac{1}{i}} e^{-2\reg \cs{2} \sum_{i=\beta_1 T+1}^{\top} \frac{1}{i}} \Esub{1}{T}{\|w_{i_0}-w^*\|^2} \\
&\leq	e^{-2\reg \cs{1} \log \frac{\beta_1 T}{i_0}} e^{-2\reg \cs{2} (\log \frac{1}{\beta_1}-\frac{1}{\beta_1 T})} \Esub{1}{T}{\|w_{i_0}-w^*\|^2} \\
&	\leq \left( \dfrac{i_0}{T} \right)^{2\reg \cs{1}} \beta_1^{2\reg (\cs{2}-\cs{1})} e^{2\reg \cs{2} /\beta_1 T} (4B^2) \notag \\
& \leq \left( \dfrac{i_0}{T} \right)^{2\reg \cs{1}} \beta_1^{2\reg (\cs{2}-\cs{1})} \left( 1+\dfrac{2\reg \cs{2}}{\beta_1 T} e^{2\reg \cs{2} /\beta_1} \right)  4B^2 \notag \\
&	= 4B^2  {i_0}^{2\reg \cs{1}} \beta_1^{2\reg (\cs{2}-\cs{1})} \dfrac{1}{T^{2\reg \cs{1}}} + \calO\left(\dfrac{1}{T^{2\reg \cs{1}+1}}\right),
\end{align*}
where the first equality follows from \eqref{eqn:sum_int_dec_geq}.%
The second inequality follows from $\|w\| \leq B$, $\|w-w'\| \leq \|w\|+\|w'\| \leq 2B$, and bounding expectation using maximum. The third  follows from \eqref{eq:bd_e_longterm}.

\paragraph{The second term in \eqref{eq:tvar_ub1}:} We bound this as follows:
\begin{align}
\Noi{1}^2 e^{-2\reg \cs{2} \sum_{i=\beta_1 T+1}^{\top} \frac{1}{i}} \sum_{i=i_0}^{\beta_1 T}\dfrac{\cs{1}^2}{i^2} e^{-2\reg \cs{1} \sum_{j=i+1}^{\beta_1 T} \frac{1}{j}} 
&\leq \Noi{1}^2 e^{-2\reg \cs{2} (\log \frac{1}{\beta_1}-\frac{1}{\beta_1 T})} \sum_{i=i_0}^{\beta_1 T}\dfrac{\cs{1}^2}{i^2} e^{-2\reg \cs{1} \log \frac{\beta_1 T}{i+1}} \notag\\
&=	\Noi{1}^2 \beta_1^{2\reg \cs{2}} e^{2\reg \cs{2}/\beta_1 T} \sum_{i=i_0}^{\beta_1 T}\dfrac{\cs{1}^2}{i^2} \left(\dfrac{i+1}{\beta_1 T} \right)^{2\reg \cs{1}} \notag\\
&=	\Noi{1}^2 \beta_1^{2\reg (\cs{2}-\cs{1})} e^{2\reg \cs{2}/\beta_1 T} \cs{1}^2 T^{-2\reg \cs{1}} \sum_{i=i_0}^{\beta_1 T}\dfrac{(i+1)^{2\reg \cs{1}}}{i^2}  \notag\\
\leq& \Noi{1}^2 \beta_1^{2\reg (\cs{2}-\cs{1})} e^{2\reg \cs{2} /\beta_1 T} \cs{1}^2 T^{-2\reg \cs{1}} \sum_{i=i_0}^{\beta_1 T} 4 (i+1)^{2\reg \cs{1}-2} \notag\\
\leq& 4\Noi{1}^2 \beta_1^{2\reg (\cs{2}-\cs{1})} \left(1+\dfrac{2\reg \cs{2}}{\beta_1 T} e^{2\reg \cs{2} /\beta_1}\right) \cs{1}^2 T^{-2\reg \cs{1}} \sum_{i=i_0+1}^{\beta_1 T+1} i^{2\reg \cs{1}-2},
\label{eq:app:2nd}
\end{align}
where the first inequality follows from \eqref{eqn:sum_int_dec_geq}, %
the second inequality follows from $(1+\frac{1}{i})^2 \leq (1+\frac{1}{1})^2=4$, and the last inequality follows from \eqref{eq:bd_e_longterm}.

Bounding summation using integral following \eqref{eqn:sum_int_dec_leq} and \eqref{eqn:sum_int_inc_leq} of Lemma \ref{lem:sum_int}, if $2\reg \cs{1} > 1$, the term on the right hand side would be in the order of $\calO(1/T)$; 
if $2\reg \cs{1}=1$, it would be $\calO(\log T/T)$; if $2\reg \cs{1}<1$, it would be
 $\calO(1/T^{2\reg \cs{1}})$. 
Therefore to minimize the bound in terms of order, we would choose $\cs{1}$ such that $2\reg \cs{1} > 1$. To get an upper bound of the summation in \eqref{eq:app:2nd}, using \eqref{eqn:sum_int_dec_leq} of Lemma \ref{lem:sum_int}, for $2\reg \cs{1}<2$,

\begin{align*}
\sum_{j=i_0+1}^{\beta_1 T+1} i^{2\reg \cs{1}-2}  = \sum_{j=i_0+1}^{\beta_1 T} i^{2\reg \cs{1}-2}  + (\beta_1 T+1)^{2\reg \cs{1}-2} \le \dfrac{(\beta_1 T)^{2\reg \cs{1}-1}}{2\reg \cs{1}-1}+ \calO(T^{2\reg \cs{1}-2}).
\end{align*}

\noindent For $2\reg \cs{1}>2$, using \eqref{eqn:sum_int_inc_leq} of Lemma \ref{lem:sum_int},
\begin{align*}
\sum_{j=i_0+1}^{\beta_1 T+1} i^{2\reg \cs{1}-2} = \sum_{j=i_0+1}^{\beta_1 T - 1} i^{2\reg \cs{1}-2} +  (\beta_1 T)^{2\reg \cs{1}-2} + (\beta_1 T+1)^{2\reg \cs{1}-2} \le 	\dfrac{(\beta_1 T)^{2\reg \cs{1}-1}}{2\reg \cs{1}-1} +  \calO(T^{2\reg \cs{1}-2}).
\end{align*}
Finally, for $2\reg \cs{1}=2$,
\begin{align*}
\sum_{j=i_0+1}^{\beta_1 T+1} i^{2\reg \cs{1}-2} 
= 	(\beta_1 T+1) - (i_0+1) +1 
=  \beta_1 T + \calO(1).
\end{align*}
Combining the three cases together, we have
\begin{align*}
\sum_{j=i_0+1}^{\beta_1 T+1} i^{2\reg \cs{1}-2} 
\le 	\dfrac{(\beta_1 T)^{2\reg \cs{1}-1}}{2\reg \cs{1}-1} +  \bigO{T^{2\reg \cs{1}-2}}.
\end{align*}
This allows us to further upper bound \eqref{eq:app:2nd}:
\begin{align*}
4\Noi{1}^2 \beta_1^{2\reg (\cs{2}-\cs{1})} \left(1+\dfrac{2\reg \cs{2}}{\beta_1 T} e^{2\reg \cs{2} /\beta_1}\right) \cs{1}^2  T^{-2\reg \cs{1}} \sum_{i=i_0+1}^{\beta_1 T+1} i^{2\reg \cs{1}-2} \\
&\hspace{-3.5in} \leq 
	4\Noi{1}^2 \beta_1^{2\reg (\cs{2}-\cs{1})} 
		\left(1+\dfrac{2\reg \cs{2}}{\beta_1 T} e^{2\reg \cs{2} /\beta_1}\right) 
		\cs{1}^2  T^{-2\reg \cs{1}}
		\left(\dfrac{(\beta_1 T)^{2\reg \cs{1}-1}}{2\reg \cs{1}-1} 
		+  \calO\left(T^{2\reg \cs{1}-2}\right)\right) \\
&\hspace{-3.5in}	= \dfrac{4\Noi{1}^2  \cs{1}^2 \beta_1^{2\reg \cs{2}-1}}{2\reg \cs{1}-1} \cdot \dfrac{1}{T}  + \calO\left(\dfrac{1}{T^2}\right) + \calO\left(\dfrac{1}{T^3}\right).
\end{align*}

\paragraph{The last term in \eqref{eq:tvar_ub1}:} We bound this as follows: 
\begin{align}
& \Noi{2}^2 \sum_{i=\beta_1 T+1}^{\top}\dfrac{\cs{2}^2}{i^2} e^{-2\reg \cs{2}  \sum_{j=i+1}^{\top} \frac{1}{j}}  \leq	\Noi{2}^2 \sum_{i=\beta_1 T+1}^{\top}\dfrac{\cs{2}^2}{i^2} e^{-2\reg \cs{2}  \log \frac{T}{i+1}} \notag\\
&=		\Noi{2}^2 \cs{2}^2 T^{-2\reg \cs{2}}\sum_{i=\beta_1 T+1}^{\top}\dfrac{(i+1)^{2\reg \cs{2}}}{i^2} \leq	4\Noi{2}^2 \cs{2}^2 T^{-2\reg \cs{2}}\sum_{i=\beta_1 T+1}^{\top}\dfrac{(i+1)^{2\reg \cs{2}}}{(i+1)^2}	\notag	\\
&=		4\Noi{2}^2 \cs{2}^2 T^{-2\reg \cs{2}}\sum_{i=\beta_1 T+2}^{\top+1} i^{2\reg \cs{2}-2},
\label{eq:app:3nd}
\end{align}
where the first inequality follows from \eqref{eqn:sum_int_dec_geq} %
and the last inequality from $(1+\frac{1}{i})^2\leq 4$.\\
If $2\reg \cs{2} \neq 1$ and $2\reg \cs{2} \leq 2$, using \eqref{eqn:sum_int_dec_leq} from Lemma \ref{lem:sum_int},
\begin{align*}
\sum_{j=\beta_1 T+2}^{T+1} i^{2\reg \cs{2}-2} \leq \dfrac{1-\beta_1^{2\reg \cs{2}-1}}{2\reg \cs{2}-1}T^{2\reg \cs{2}-1}.
\end{align*}
If $2\reg \cs{2} > 2$, using \eqref{eqn:sum_int_inc_leq} from Lemma \ref{lem:sum_int}, 
\begin{align*}
\sum_{j=\beta_1 T+2}^{T+1} i^{2\reg \cs{2}-2}
&= 		\sum_{j=\beta_1 T}^{T-1} i^{2\reg \cs{2}-2} +  T^{2\reg \cs{2}-2} + (T+1)^{2\reg \cs{2}-2}
	 - (\beta_1 T+1)^{2\reg \cs{2}-2} - (\beta_1 T)^{2\reg \cs{2}-2}\\
&= 		\dfrac{1-\beta_1^{2\reg \cs{2}-1}}{2\reg \cs{2}-1}T^{2\reg \cs{2}-1}  +  \bigO{T^{2\reg \cs{2}-2}}.
\end{align*}
If $2\reg \cs{2} = 2$,
\begin{align*}
\sum_{j=\beta_1 T+2}^{T+1} i^{2\reg \cs{2}-2}
= 		\sum_{j=\beta_1 T+2}^{T+1} 1
= 		(1-\beta_1)T.
\end{align*}
In all three cases we have
\begin{align*}
\sum_{j=\beta_1 T+2}^{T+1} i^{2\reg \cs{2}-2} \leq
\dfrac{1-\beta_1^{2\reg \cs{2}-1}}{2\reg \cs{2}-1}T^{2\reg \cs{2}-1}  +  \bigO{T^{2\reg \cs{2}-2}}.
\end{align*}
Then \eqref{eq:app:3nd} can be further upper bounded for $2\reg \cs{2} \neq 1$
\begin{align}
4\Noi{2}^2 \cs{2}^2 T^{-2\reg \cs{2}} \sum_{i=\beta_1 T+2}^{\top+1} i^{2\reg \cs{2}-2} 
\le 	4\Noi{2}^2 \dfrac{\cs{2}^2 (1 - \beta_1^{2\reg \cs{2}-1})}{2\reg \cs{2}-1} \cdot \dfrac{1}{T} + \calO\left(\dfrac{1}{T^2}\right). \label{eq:app:3nd:neq1}
\end{align}
If $2\reg \cs{2} = 1$, we have
\begin{align*}
\sum_{j=\beta_1 T+2}^{T+1} i^{2\reg \cs{2}-2} =		\sum_{j=\beta_1 T+1}^{T} i^{-1} - (\beta_1 T+1)^{-1} + (T+1)^{-1} \leq 		\log \dfrac{1}{\beta_1},
\end{align*}
and then
	\[
			4\Noi{2}^2 \cs{2}^2 T^{-2\reg \cs{2}}\sum_{i=\beta_1 T+2}^{\top+1} i^{2\reg \cs{2}-2}
	\le 4\Noi{2}^2 \cs{2}^2 \log \dfrac{1}{\beta_1} \cdot \dfrac{1}{T}.
	\]
which is basically taking the limit as $2\reg \cs{2} \rightarrow 1$ of the highest order term of \eqref{eq:app:3nd:neq1}.

Therefore the summation of the three terms is of order $\calO(\frac{1}{T})$ (from the second and third terms), and the constant in the front of the highest order term takes on one of two values:
\begin{enumerate}
\item If $2\reg \cs{2} \neq 1$, 
\begin{equation*}
4\Noi{1}^2 \dfrac{\cs{1}^2 \beta_1^{2\reg \cs{2}-1}}{2\reg \cs{1} - 1}+ 
4\Noi{2}^2 \dfrac{\cs{2}^2 (1 - \beta_1^{2\reg \cs{2}-1})}{2\reg \cs{2}-1}.
\end{equation*}
\item If $2\reg \cs{2} = 1$, 
\begin{equation*}
4\Noi{1}^2 \dfrac{\cs{1}^2 \beta_1^{2\reg \cs{2}-1} }{2\reg \cs{1} - 1}+ 
4\Noi{2}^2 \cs{2}^2 \log \dfrac{1}{\beta_1}.
\end{equation*}
\end{enumerate}
\end{proof}

\subsubsection{Proof of Lemma \ref{lem:largenoi1}}
\begin{proof} (Of Lemma~\ref{lem:largenoi1})
Omitting the constant terms and setting $k_1 = 2\reg\cs{1}, k_2=2\reg\cs{2}$, we can re-write \eqref{eq:orderopt_t} as $1/T$ times
\begin{align}\label{eq:orderopt_t_2}
Q(k_1, k_2) = & \Noi{1}^2 \frac{\beta_1^{k_2 - 1} k_1^2}{k_1 - 1}  + \Noi{2}^2 \frac{(1 - \beta_1^{k_2 - 1}) k_2^2}{k_2 - 1},
\end{align} 
with $k_1^* = 2\reg\cs{1}^*=2$.\\
Observe that in this case, $k_2^* \geq 2$. Let $x = k_2 - 1$; then $x \geq 1$. Plugging in $k_1^* = 2$, we can re-write~\eqref{eq:orderopt_t_2} as
\begin{equation}\label{eq:xq}
Q(x) = 4 \Noi{1}^2 \beta_1^{x} + \Noi{2}^2 (1 - \beta_1^x) \left(x + \frac{1}{x} + 2 \right).
\end{equation}
Taking the derivative, we see that
\begin{align}\label{eq:xdervq}
Q'(x) =& - 4 \Noi{1}^2 \beta_1^{x} \log(1/\beta_1) + \Noi{2}^2 (1 - \beta_1^x) \left(1 - \frac{1}{x^2} \right) 
	+ \Noi{2}^2 \left(x + \frac{1}{x} + 2\right) \beta_1^x \log(1/\beta_1).
\end{align}
Suppose 
	\[
	l = \frac{2\log(\Noi{1}/\Noi{2}) + \log\log(1/\beta_1)}{\log (1/\beta_1)}.
	\]
Observe that $\beta_1^{l} \log(1/\beta_1) = \frac{\Noi{2}^2}{\Noi{1}^2}$. Plugging $x = l$ in to~\eqref{eq:xdervq}, the first term is $-4 \Noi{2}^2$, the second term is at most $\Noi{2}^2$, and the third term is at most $\frac{\Noi{2}^4}{\Noi{1}^2} (l + \frac{1}{l} + 2)$. Observe that for any fixed $\beta_1$, for large enough $\Noi{1}/\Noi{2}$, $l \geq 1$. Thus, the right hand side of~\eqref{eq:xdervq} is at most: $- 4 \Noi{2}^2 + \Noi{2}^2 + \frac{\Noi{2}^4}{\Noi{1}^2}(l + 3)$. For fixed $\beta_1$, $l$ grows logarithmically in $\Noi{1}/\Noi{2}$, and hence, for large enough $\Noi{1}/\Noi{2}$, $\frac{\Noi{2}^2 (l + 3)}{\Noi{1}^2}$ will become arbitrarily small. Therefore, for large enough $\Noi{1}/\Noi{2}$, $Q'(l) < 0$.
\\
Suppose 
	\[
	u = \frac{2 \log(4 \Noi{1}/\Noi{2}) + \log \log(1/\beta_1)}{\log(1/\beta_1)}.
	\] 
Observe that $\beta_1^u \log(1/\beta_1) = \frac{\Noi{2}^2}{16\Noi{1}^2}$. Plugging in $x = u$ to~\eqref{eq:xdervq}, the first term reduces to $-\frac{1}{4} \Noi{2}^2$, the second term is $\Noi{2}^2 (1 - \beta_1^u)(1 - \frac{1}{u^2})$, and the third term is $\geq 0$. Observe that as $\Noi{1}/\Noi{2} \rightarrow \infty$ with $\beta_1$ fixed, $\beta_1^u \rightarrow 0$ and $1/u^2 \rightarrow 0$. Thus, for large enough $\Noi{1}/\Noi{2}$, $\Noi{2}^2 (1 - \beta_1^u)(1 - \frac{1}{u^2}) \rightarrow \Noi{2}^2$, and therefore $Q'(u) > 0$. Thus, $Q'(x) = 0$ somewhere between $l$ and $u$ and the first part of the lemma follows.
\\
Consider 
	\[
	x = \frac{2 \log(m \Noi{1}/\Noi{2}) + \log \log(1/\beta_1)}{\log(1/\beta_1)}
	\]
with $1\le m\le 4$. The first term of \eqref{eq:xq} is always positive. As for the second term, $x+\frac{1}{x}+2 \geq x$ for positive $x$ and $\beta_1^x = \frac{\Noi{2}^2}{m^2 \Noi{1}^2} \frac{1}{\log(1/\beta_1)}$ is small when $\Noi{1}/\Noi{2}$ is sufficiently large. Therefore for sufficiently large $\Noi{1}/\Noi{2}$, we have $\Noi{2}^2 (1 - \beta_1^x) (x + \frac{1}{x} + 2) \geq \frac{\Noi{2}^2}{2} x$, and thus $Q(x) \geq \frac{\Noi{2}^2}{2} x$, which gives the lower bound. And plugging in $x=l$ gives the upper bound.

\end{proof}

\subsubsection{Proof of Lemma \ref{lem:largenoi2}}
\begin{proof}(Of Lemma~\ref{lem:largenoi2})
Let $k_2 = \epsilon$; then $\epsilon \geq 0$. Plugging in $k_1^* = 2$, we can re-write~\eqref{eq:orderopt_t_2} as
\begin{equation}\label{eq:xq2}
Q(\epsilon) = 4 \Noi{1}^2 \beta_1^{\epsilon - 1} + \Noi{2}^2 (1 - \beta_1^{\epsilon - 1}) \left(-1 + \epsilon + \frac{1}{-1 + \epsilon} + 2 \right).
\end{equation}
Taking the derivative, we obtain the following:
\begin{align}\label{eq:xdervq2}
Q'(\epsilon) 
 = & - 4 \Noi{1}^2 \beta_1^{\epsilon - 1} \log(1/\beta_1) + \Noi{2}^2 (1 - \beta_1^{\epsilon - 1}) (1 - \frac{1}{(1 - \epsilon)^2}) - \frac{\Noi{2}^2 \epsilon^2}{1 - \epsilon} \beta_1^{\epsilon - 1} \log(1/\beta_1)  \nonumber \\
 = &  - \beta_1^{\epsilon - 1} \log(1/\beta_1) \left(4 \Noi{1}^2 + \frac{\Noi{2}^2 \epsilon^2}{1 - \epsilon}\right) + \Noi{2}^2 (\beta_1^{\epsilon - 1} - 1) \left( \frac{1}{(1 - \epsilon)^2} - 1 \right) \nonumber
\\
 = &  - \beta_1^{\epsilon - 1} \log(1/\beta_1) \left(4 \Noi{1}^2 + \frac{\Noi{2}^2 \epsilon^2}{1 - \epsilon}\right) + \Noi{2}^2 (\beta_1^{\epsilon - 1} - 1) \frac{\epsilon(2-\epsilon)}{(1 - \epsilon)^2}.
\end{align}
For $\epsilon = \frac{\Noi{1}^2}{\Noi{2}^2}\le 1$, using $1 - \beta_1^{1 - \epsilon} \leq (1 - \epsilon) \log(1/\beta_1)$ and $\beta_1^{\epsilon - 1} - 1 = (1 - \beta_1^{1 - \epsilon}) \beta^{\epsilon-1}$, this is at most:
	\[ 
	-\beta_1^{\epsilon - 1} \log(1/\beta_1) 
		\left( 4 \Noi{1}^2 + \frac{\Noi{2}^2 \epsilon^2}{1 - \epsilon} 
				- \frac{\Noi{2}^2 \epsilon (2 - \epsilon)}{ 1 - \epsilon } 
		\right) 
	= -2 \Noi{1}^2 \beta_1^{\epsilon - 1} \log(1/\beta_1).
	\]
Thus, at $l = \frac{\Noi{1}^2}{\Noi{2}^2}$, $Q'(l) < 0$.
\\
Moreover, for $\epsilon \in [0, \frac{1}{2}]$, $1 - \beta_1^{1 - \epsilon} \geq \beta_1 (1 - \epsilon) \log (1/\beta_1)$. Therefore, $Q'(\epsilon)$ is at least:
\begin{align*}
Q'(\epsilon) &\ge -\beta_1^{\epsilon - 1} \log(1/\beta_1) \left(4 \Noi{1}^2 + \frac{\Noi{2}^2 \epsilon^2}{1 - \epsilon}\right) + \Noi{2}^2 \beta_1^{\epsilon} \log (1 / \beta_1) \frac{\epsilon (2 - \epsilon)}{1 - \epsilon} \\
&\geq  \beta_1^{\epsilon - 1} \log(1/\beta_1) \left( \frac{\Noi{2}^2 \beta \epsilon (2 - \epsilon)}{1 - \epsilon} - 4 \Noi{1}^2 - \frac{\Noi{2}^2 \epsilon^2}{1 - \epsilon} \right).
\end{align*}
Let $u = \frac{8 \Noi{1}^2}{\beta_1 \Noi{2}^2}$; suppose that $\Noi{2}/\Noi{1}$ is large enough such that $u \leq \beta_1/4$. Then, $u (2 - u) \beta_1 - u^2 \geq \frac{15 u \beta_1}{16}$, and 
	\[ 
	\frac{\Noi{2}^2 (u (2 - u) \beta_1 - u^2)}{1 - u} 
	\geq \frac{15 \Noi{2}^2 u \beta_1}{16 (1 - \beta_1)}  
	\geq \frac{15 \Noi{1}^2}{2 (1 - \beta_1)} \geq 5 \Noi{1}^2. \]
Therefore, $Q'(u) > 0$, and thus $Q(\epsilon)$ is minimized at some $\epsilon \in [l, u]$. 
\\
For the second part of the lemma, the upper bound is obtained by plugging in $\epsilon = \frac{\Noi{1}}{\Noi{2}}$. For the lower bound, observe that for any $\epsilon \in [ l, u]$, $Q(\epsilon) \geq 4 \Noi{1}^2 \beta_1^{u - 1} \geq 4 \Noi{1}^2 \beta_1^{\Noi{2}^2/\beta \Noi{1}^2 - 1}$.  
\end{proof}

\bibliography{privacy}
\bibliographystyle{plainnat}

\end{document}